\documentclass[twocolumn,10pt]{IEEEtran}
\IEEEoverridecommandlockouts
\usepackage{textcomp}
\usepackage{comment}
\includecomment{toexclude}
\usepackage{amsmath,amsfonts,amssymb}
\ifCLASSOPTIONcompsoc
\usepackage[nocompress]{cite}
\else
\usepackage{cite}
\fi
\usepackage[utf8]{inputenc}
\usepackage{verbatim}
\usepackage{graphicx}
\usepackage{xcolor}
\usepackage{subcaption}
\usepackage{amssymb}
\usepackage{amsthm}
\usepackage{lineno}
\usepackage{multirow}
\usepackage{booktabs}
\usepackage[outdir=./]{epstopdf}
\usepackage{color}
\usepackage{multicol}
\usepackage{bm}
\usepackage{esint}
\usepackage{diagbox}
\usepackage{booktabs} 
\usepackage{stfloats,siunitx}
\usepackage{algcompatible}
\usepackage{empheq}

\usepackage[ruled]{algorithm2e}
\usepackage{array}

\allowdisplaybreaks[4]
\def\BibTeX{{\rm B\kern-.05em{\sc i\kern-.025em b}\kern-.08em
		T\kern-.1667em\lower.7ex\hbox{E}\kern-.125emX}}
\allowdisplaybreaks[4]	

\definecolor{mygreen}{RGB}{32,178,170}  
\definecolor{mygolden}{RGB}{255,140,0} 

\makeatletter
\newcommand{\sublabel}[1]{\protected@edef\@currentlabel{\thefigure(\thesubfigure)}\label{#1}}

\makeatother

\newcommand{\A}{\mathbf{A}}

\newcommand{\diag}{\text{diag}}

\newtheorem{theorem}{Theorem}
\newtheorem{remark}{Remark}

\newtheorem{lemma}{Lemma}

\newtheorem{corollary}{Corollary}
\newtheorem{definition}{Definition}

\newtheorem{assumption}{Assumption}

\begin{document}
	\captionsetup{font={small}}
	
	\title{
		Amplitude-Varying Perturbation for Balancing Privacy and Utility in Federated Learning
	}
	\author{Xin Yuan,~\IEEEmembership{Member, IEEE}, Wei Ni,~\IEEEmembership{Senior Member, IEEE}, Ming Ding,~\IEEEmembership{Senior Member, IEEE}, \\
		Kang Wei,~\IEEEmembership{Member, IEEE}, Jun Li,~\IEEEmembership{Senior Member, IEEE}, and H. Vincent Poor,~\IEEEmembership{Life Fellow, IEEE}
		\thanks{
			X. Yuan, W. Ni, and M. Ding are with Data61, CSIRO, Sydney, Australia (e-mail: {xin.yuan, wei.ni, ming.ding}@data61.csiro.au).
			
			K. Wei and J. Li are with the School of Electrical and Optical Engineering, Nanjing University of Science and Technology, Nanjing 210094, China (e-mail: {kang.wei, jun.li}@njust.edu.cn).
			
			H. V. Poor is with the Department of Electrical and Computer Engineering, Princeton University, Princeton, NJ 08544 USA (e-mail: poor@princeton.edu).
			
	}}

	\markboth{ACCEPTED BY IEEE TRANSACTIONS ON INFORMATION FORENSICS \& SECURITY}
	{YUAN \MakeLowercase{\textit{et al.}}: Amplitude-Varying Perturbation for Balancing Privacy and Utility in Federated Learning}
	\maketitle

	\begin{abstract}
		While preserving the privacy of federated learning (FL), differential privacy (DP) inevitably degrades the utility (i.e., accuracy) of FL due to model perturbations caused by DP noise added to model updates. 
		Existing studies have considered exclusively noise with persistent root-mean-square amplitude and overlooked an opportunity of adjusting the amplitudes to alleviate the adverse effects of the noise.
		This paper presents a new DP perturbation mechanism with a time-varying noise amplitude to protect the privacy of FL and retain the capability of adjusting the learning performance. 
		Specifically, we propose a geometric series form for the noise amplitude and reveal analytically the dependence of the series on the number of global aggregations and the $(\epsilon,\delta)$-DP requirement.
		We derive an online refinement of the series to prevent FL from premature convergence resulting from excessive perturbation noise.
		Another important aspect is an upper bound developed for the loss function of a multi-layer perceptron (MLP) trained by FL running the new DP mechanism. 
		Accordingly, the optimal number of global aggregations is obtained, balancing the learning and privacy. 
		Extensive experiments are conducted using MLP, supporting vector machine, and convolutional neural network models on four public datasets. The contribution of the new DP mechanism to the convergence and accuracy of privacy-preserving FL is corroborated, compared to the state-of-the-art Gaussian noise mechanism with a persistent noise amplitude.
		
	\end{abstract}
	
	
	\section{Introduction}\label{sec-intro}
	Federated learning (FL) 
	trains machine learning (ML) models at individual devices without the need to surrender any sensitive raw data of the devices to central servers~\cite{Li2020Federated}.
	It provides 
	an effective means of model training  without directly leaking private information~\cite{yang2019federated}.
	Despite its significant potential for privacy protection, 
	there are still risks of revealing sensitive
	information in the individual models uploaded to an aggregator (e.g., a central server) on each aggregation round of FL.
	In particular,  the local models learned from the respective local datasets can be reverse-engineered (e.g., by a curious central server) to extract private information~\cite{Milad2019Comprehensive, toch2018towards}. 
	Notably, Shokri~{\textit{et al.}}~\cite{Shokri2017Membership} demonstrated that private information about local datasets can be derived from trained local models. 
	Moreover, model inversion attacks have been shown to be able to extract private information by using black-box attacks to predict models~\cite{fredrikson2015model, wang2015regression}. 
	
	Privacy-preserving FL is a promising method for solving the above challenges~\cite{wu2021incentivizing,geyer2017differentially, lu2019differentially, bhowmick2018protection,Ma2020On}. 
	It incorporates privacy techniques into distributed ML frameworks to deliver a provable guarantee of privacy protection~\cite{bassily2014private, Abadi2016Deep, Chollet2017Xception,Wei2022User, Wei2022Low}. 
	A Trusted Execution Environment (TEE) is a secure and isolated computing environment that uses hardware and software encryption to protect sensitive data and ensure the accuracy of computations. It guarantees the confidentiality and integrity of an individual client's application, even in an untrusted environment~\cite{nguyen2022federated}. In FL, TEEs can be adopted by clients for local training and/or by central servers for secure aggregation of local updates to prevent attacks on the models or gradients~\cite{mo2021ppfl,prakash2022Secure}.
	Differential privacy (DP)~\cite{mironov2017renyi} is the \textit{de facto} privacy mechanism that has been increasingly studied, 
	including $\epsilon$-DP, $(\epsilon, \delta)$-DP, R{\'e}nyi DP, and $(\alpha, \epsilon)$-R{\'e}nyi DP~\cite{zhu2017differentially, acs2018differentially, yu2019differentially, barthe2011information, mironov2017renyi}.
	In~\cite{yu2019differentially}, 
	a time-varying noise perturbation mechanism was proposed, 
	where a time-decaying noise was added to the model parameters.
	In~\cite{barthe2011information}, 
	information-theoretic bounds were derived to establish a connection between information leakage and DP.

	DP mechanisms have been increasingly integrated into FL, aiming to learn a secure global model while providing privacy guarantees for local datasets. 
	This allows the clients to efficiently train their local models with privacy protected according to local settings.
	Truex~{\textit{et al.}}~\cite{Truex2020Federated} proposed an FL system with local differential privacy (LDP) to ensure data privacy. 
	The system can perform LDP-based perturbation on model parameter updating and sharing, according to the local privacy level.
	In~\cite{Wei2020Federated}, a framework based on the DP was proposed to prevent information leakage by injecting noise to protect the privacy of the local model parameters.
	In~\cite{Zhao2020Local}, three LDP mechanisms were developed to preserve privacy in different data analysis tasks. 
	The LDP mechanisms were integrated into FL to predict traffic status, alleviate privacy threats, and reduce communications in crowd-sourcing applications. 
	The above studies combined the local DP mechanisms with FL to address privacy issues. 
	However, considerable noises are needed to perturb the local model parameters, reducing the efficiency and accuracy of FL.

	The studies discussed above, i.e.,~\cite{Truex2020Federated,Wei2020Federated,Zhao2020Local}, have considered time-invariant DP noise perturbations; in other words, the variance of the DP perturbation noise remains persistent among global aggregations. This could require a long training or convergence time and degrade the learning performance, such as validation accuracy~\cite{frisk2021super}. On the other hand, there is clearly a potential to allow the variance of the DP perturbation noise to be adaptively configured and changed over different global aggregations, hence improving the learning performance of FL without compromising the privacy protection level. In particular, a small perturbation noise in the early stage of an FL process is expected to benefit convergence~\cite{cheng2022differentially}, as also observed experimentally in this paper. 
	
	This paper presents a new $(\epsilon, \delta)$-DP amplitude-varying perturbation mechanism with a meticulously designed time-varying root-mean-square amplitude (or amplitude for short) of the perturbation noise to strike a balance between the privacy protection and the utility (i.e., loss and/or accuracy) of FL. The DP noise can be further adjusted online to combat the potential degradation of the utility, achieving effective learning while preserving privacy.
	
	The key contributions are listed as follows.
	\begin{itemize}
		\item We design the new $(\epsilon, \delta)$-DP perturbation mechanism, where the variance (i.e., the square of the amplitude) of the DP noise is a geometric series changing over the global aggregations of the FL to provide privacy guarantees. 
		
		\item 
		By privacy analysis, we derive the variance of the DP noise given the global aggregation number and privacy protection level of the FL. We also design the online adjustment of the variance and global aggregation rounds in the face of a model degradation of the FL. 
		
		\item An upper bound is derived for the loss of a multi-layer perceptron (MLP) model trained by FL running the new DP mechanism, establishing an analytical trade-off between the loss and privacy. 
		
		\item Based on the upper bound, an optimal number of global aggregations is identified to achieve the best utility of the FL and satisfy the privacy requirement. 
		
	\end{itemize}
	Extensive experiments based on an MLP model show that the new DP mechanism with time-varying noise variance  converges faster to better learning accuracy for a given privacy level, compared to the state-of-the-art noise perturbation with persistent variance, i.e., the Gaussian noise mechanism~\cite{Wei2020Federated}.
	Moreover, the new mechanism is readily applicable to other deep neural network (DNN) models, such as support vector machine (SVM) and convolutional neural network (CNN).

	The remainder of this paper is arranged as follows. The system and threat models are provided in Section~\ref{sec-definition}. In Section~\ref{sec-problem}, we elaborate on the new DP mechanism, analyze its sensitivity and privacy, and the time-varying variance of the DP noise. Section~\ref{sec-convergence analysis} derives the convergence upper bound of FL running the new DP mechanism, and the optimal number of global aggregations to achieve both convergence and privacy.
	Experimental results are discussed in Section~\ref{sec-results}, followed by concluding remarks in Section~\ref{sec-conclusion}.
	
	\textit{Notation:} $(\cdot)^H$, $(\cdot)^{\top}$ and $(\cdot)^c$ are the Hermitian transpose, transpose, and conjugate of a matrix/vector, respectively. $|\cdot|$ takes element-wise absolute values. $\left\|\cdot \right\| $ denotes $\ell_2$-norm. $(\A)_{n,m}$ and $(\A)_{\cdot,m}$ stand for the $(n,m)$-th element and the $m$-th column of the matrix $\A$, respectively. 
	$\diag\{a_n\}$ stands for a diagonal matrix with $a_n,\forall n$ along its diagonal. The notation used is collated in Tab.~\ref{table_notations}.
	
	\begin{table}[!t]\small
		\caption{Summary of Notation}
		\begin{center}
			\begin{tabular}{ll}
				\toprule[1.5pt]
				Notation & Description \\
				\hline
				$\cal M$ & A random DP mechanism \\
				$\cal D$, ${\cal D}'$ & Adjacent datasets \\
				$\epsilon$, $\delta$ & DP requirement\\
				${\cal D}_k$ & Dataset held by user ${\cal C}_k$\\
				$\nabla F(\cdot)$ & Gradient of a function $F(\cdot)$\\
				$U$ & Total number of users \\
				$K$ & Number of chosen users\\
				$t$ & Iteration index\\
				$T$ & Total number of iterations\\
				$\tau$ & Number of local training iterations between two \\
				& global aggregations\\
				$M$ & Maximum number of global aggregations\\
				$\bm{\omega}$ & Parameters of the model\\
				$F(\bm{\omega})$ & Global loss function\\
				$f_k(\bm{\omega})$ & Loss function of the $k$-th user\\
				${\bm{\omega}}_k(t)$ & Local model parameters of the $k$-th user\\
				$\tilde{\bm{\omega}}_k(t)$ & Local model parameters of the $k$-th user\\ & after adding noises\\
				${\bm{\omega}}(m)$ & The aggregated model parameters for the $m$-th \\
				& global aggregation, $m=0,1,\cdots,M$\\
				${\bm{\omega}}^*$ & Optimal model parameters\\
				\toprule[1.5pt]
			\end{tabular}
		\end{center}
		\label{table_notations}
	\end{table}

	\section{System Model}\label{sec-definition}
	This section introduces the system and threat models of the considered FL system. 
	
	\begin{figure}[tb]
		\centering
		\includegraphics[width=0.9\columnwidth]{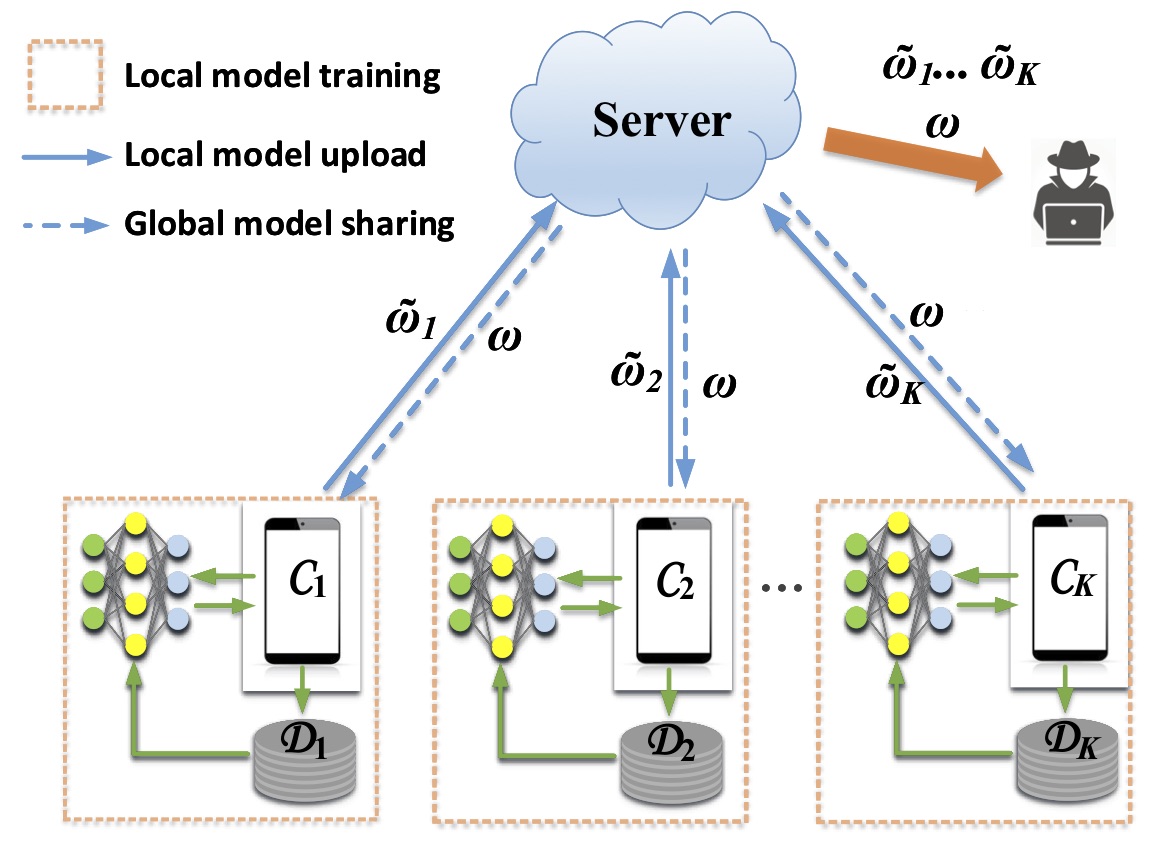}
		\caption{An FL model with an honest server and an attacker trying to eavesdrop on or capture private information from both the local model upload and global model sharing process.}
		\label{fig-training-model}
	\end{figure}

	\subsection{Federated Learning}
	The considered FL system consists of $U$ users collected by the set ${\cal U} = \left\lbrace1,\cdots, U \right\rbrace $ and an honest-but-curious parameter server, as illustrated in Fig.~\ref{fig-training-model}. ${\cal D}_i$ is the local dataset at user $i \in {\cal U}$. 
	$|{\cal D}|=\sum_{i \in {\cal U}} |{\cal D}_i|$ is the collection of all data samples.
	The server wishes to train a global model on the datasets residing at the $U$ users, by minimizing the global loss function, $F(\omega)$:
	\begin{equation}\label{eq-problem}
	{\bm \omega}^* \triangleq \arg \underset{\bm \omega}{\min} F({\bm \omega}),
	\end{equation}
	where $\bm{\omega}$ stands for the model parameter, 
	$\bm{\omega}^*$ stands for the optimal model parameter minimizing the global loss function, 
	and  ${\cal K}  \in {\cal U}$ is a set of $K$ randomly chosen  users from  $\cal U $ between two consecutive global aggregations.
	$F({\bm{\omega}}) = \sum_{k \in {\cal K}} p_k f_k({\bm \omega})$, with 
	$f_k({\bm{\omega}})$ being the loss function of the $k$-th user in $\cal K$, and $p_k = {|{\cal D}_k|}/{\sum_{k\in {\cal K}}|{\cal D}_k|}$ with $\sum_{k \in {\cal K}} p_k =1$.
	
	We consider using DP to prevent the privacy leakage of a distributed gradient descent-based FL system. 
	$M$ is the maximum number of global aggregations. $\tau$ is the number of local update iterations between two consecutive aggregations. $T$ is the total number of iterations. $T = M \tau$. 
	The FL and DP operations are summarized in \textbf{Algorithm~\ref{algo_DGD}}, where ${\bm \omega}(m)$ denotes the global model parameter obtained at the $m$-th global model aggregation with the initial global model ${\bm \omega}(0)$, ${\bm \omega}_k(t)$ stands for the local model parameter of the $k$-th user at the $t$-th iteration,
	and $q$ is the ratio of selected and participating users (i.e., $K = |{\cal K}| = qU$). $t \in [0,T]$ is the iteration index. 
	For the $k$-th user, the local model is updated by
	\begin{equation}\label{eq_local_SGD}
	\bm{\omega}_k(t+1) = {\bm{\omega}}_k(t) - \eta \nabla f_k ({\bm{\omega}}_k(t)),
	\end{equation} 
	where $\eta$ is the step size. The $k$-th user clips the local model parameter ${\bm \omega}_k(t)$ with a pre-determined threshold $C$, i.e., $\left\|\bm{\omega}_k(t+1) \right\| \leq C$. 
	
	After every $\tau$ local updates (or iterations), e.g., the $\tau m$-th iteration, the users inject the DP noises into their local models and upload the perturbed local models, denoted by $\tilde{\bm{\omega}}_k(\tau m)$, to the parameter server for the $m$-th global aggregation: 
	\begin{equation}\label{eq-noised-parameter}
	\tilde{\bm{\omega}}_k(\tau m) = {\bm{\omega}}_k (\tau m) + {\cal N}(0,\varTheta(m) {\bm I}), \; m \in [1, M],
	\end{equation}	
	where $\varTheta(m)$ is the variance of the DP noise added to the local models for the  $m$-th global aggregation.
	
	The global model is the aggregation of the DP-perturbed local models.
	At the $m$-th aggregation, the global model parameter is  
	\begin{equation}\label{eq_global_weight}
	\bm{\omega}(m) = \sum_{k \in {\cal K}} p_k \tilde{\bm{\omega}}_k(\tau m).
	\end{equation}
	The parameter server broadcasts $\bm{\omega}(m)$ to all users. The users start the next training iteration, i.e., $t=\tau m +1$, by setting ${\bm{\omega}}_k(t)={\bm{\omega}}(m)$.
	By executing Algorithm~\ref{algo_DGD} for $M$ global model aggregations (or $T$ iterations), we obtain the optimal model parameter $\bm{\omega}^*$ achieving the minimum global loss function.
	\begin{algorithm}
		\begin{algorithmic}[1]
			\caption{Distributed gradient descent FL perturbed with time-varying DP noises}
			\label{algo_DGD}
			\LinesNumbered
			\STATE {\textbf{Input:} $\tau$, $T$, $\bm{\omega}(0)$, $\epsilon$, $\delta$, and $C$.}
			\STATE {\textbf{Output:} $\bm{\omega}(T)$.}
			\STATE {Initialize: $t=0$ and $\bm{\omega}(0)$\;}
			\WHILE{$t \leq T$}
			\Statex{\% Local model update\;}
			\WHILE{$k \in {\cal K}$}
			\STATE{Update the local parameters:\\
				\qquad $~~~~\bm{\omega}_k(t+1) = {\bm{\omega}}_k(t) - \eta \nabla f_k ({\bm{\omega}}_k(t))$\;}
			\STATE{Clip the local parameters:\\
				\qquad $~~~~\bm{\omega}_k(t+1) \leftarrow \bm{\omega}_k(t+1) / \max\left(1,\frac{\left\|\bm{\omega}_k(t+1) \right\| }{C} \right)$\;}
			\ENDWHILE
			\Statex{\% Global model aggregation\;}
			\IF{$(t + 1)$ is an integer multiple of $\tau$}
			\STATE $~~~~~m = \left\lfloor\frac{t + 1}{\tau} \right\rfloor$\;
			\STATE {Produce the time-varying DP noises to perturb} 
			\STATE {the local model parameters:\\
				\qquad $~~~~~\tilde{\bm{\omega}}_k(t+1) = {\bm{\omega}}_k (t+1) + {\cal N}(0, \varTheta(m) {\bm I}),\,\forall k$\;}
			\STATE{Update the global parameters:\\
				\qquad $~~~~~\bm{\omega}(m) = \sum_{k \in {\cal K}} p_k \tilde{\bm{\omega}}_k(t+1) $\;}
			\STATE{Update the local model parameters: \\
				\qquad$~~~~~{\bm{\omega}}_k(t+1) = {\bm{\omega}}(m), \; \forall k \in {\cal K}$.}
			\ENDIF
			\STATE{$t \leftarrow t+1$;}
			\ENDWHILE
		\end{algorithmic}
	\end{algorithm}
	
	\subsection{Threat Model}
	Suppose that the parameter server is honest but curious in the considered FL system. External attackers attempt to obtain confidential information of the users. Although the users can store and train their data locally in FL, the local model updates shared between the users and server can potentially compromise the privacy of the users, e.g., under inference attacks at the learning phase~\cite{nasr2019comprehensive} and model-inversion attacks at the testing phase~\cite{fredrikson2015model}.
	The attackers can hijack the private information by analyzing the global model parameters broadcast by the parameter server.
	
	Assume that the attacker has a dataset ${\cal D}_{a}$ overlapping with user $k$'s local dataset ${\cal D}_k$ and 
	attempts to find the subset of its dataset, $ {\cal D}'_{a} \in {\cal D}_{a}$, that is 
	the most likely used for the training of user $k$'s model, i.e., $\max_{{\cal D}'_{a} \in {\cal D}_a} \Pr\left[{\cal D}'_{a} \in {\cal D}_{k} | \bm{\omega}_k \right]$.
	Here, $\Pr[{\cal D}'_{a} \in {\cal D}_{k} | {\bm{\omega}}_k]$ is the probability that the subset ${\cal D}'_{a}$ belongs to ${\cal D}_k$, given user $k$'s model $\bm{\omega}_k$.
	It is possible that an attacker has an overlapping dataset with some participating nodes in an FL setting, which is a common assumption made in the literature when analyzing FL systems, e.g.,~\cite{nasr2019comprehensive}.

	\section{Proposed DP Perturbation with Varying Noise Variance}\label{sec-problem}
	In this section, we delineate the proposed DP mechanism with time-varying noise perturbation in the considered FL system, and analyze the sensitivity and privacy of the mechanism.
	
	\subsection{Definition of DP} 
	For an $\left(\epsilon, \delta\right)$-DP mechanism, 
	the privacy is parametrized by a requirement specified using $\epsilon$ and~$\delta$. Here,
	$\epsilon > 0$ specifies the difference beyond which the outputs on two adjacent datasets ${\cal D}$ and ${\cal D}'$ can be differentiated{\footnote{Two datasets, ${\cal D}$ and ${\cal D}'$, are adjacent if ${\cal D}'$ can be built by inserting an example to, or discarding an example from, ${\cal D}$.}}. $\delta$ is the probability with which the ratio between the probabilities of ${\cal D}$ and ${\cal D}'$ after DP noises are added is no smaller than $\exp(\epsilon)$~\cite{mcsherry2007mechanism}. The definition of the $\left( \epsilon, \delta\right)$-DP is provided below.
	
	\begin{definition}[$\left( \epsilon, \delta\right)$-DP~\cite{Abadi2016Deep}]
		A random mechanism $\cal M$: ${\cal X}\to {\cal R}$ that has a domain of $\cal X$ and a range of $\cal R$ meets $\left( \epsilon, \delta\right)$-DP, as long as 
		\begin{equation}
		\Pr \left[{\cal M}(\cal D) \in {\cal S} \right] \leq e^{\epsilon} \Pr \left[{\cal M}({\cal D}')  \in {\cal S}\right] + \delta,
		\end{equation}
		for a measurable set ${\cal S} \subseteq {\cal R}$ and adjacent datasets ${\cal D}, {\cal D}' \in {\cal X}$.	
	\end{definition}
	
	\subsection{Proposed Time-varying Perturbation Noise Variance } 
	
	We propose that the variance of the DP noise, i.e., $\varTheta(m)$, added to the local model parameters changes (increases or decreases), with the increasing number of global aggregations. 
	On the one hand, for a required privacy level (e.g., $(\epsilon, \delta)$ in the context of DP), adding a smaller or stronger noise at the beginning of the training can speed up convergence~\cite{cheng2022differentially}.
	On the other hand, more global aggregations in the training process result in worse privacy leakage since an adversary can observe more information exposed in the global aggregations and related to the training datasets, according to~\cite{Wei2020Federated}. 
	
	We design that $\varTheta(m)$ is a geometric series, and the noise variance added in the $m$-th global aggregation is calculated as 
	\begin{equation}
	\varTheta (m) = \vartheta^{m -1}\sigma^2,\; m = 1,\cdots,M,
	\end{equation}
	where $\sigma^2$ is the initial noise variance; $\vartheta>0$ is the scaling factor of the series. 
	When $\vartheta = 1$, the noise variance $\varTheta (m)$ remains unchanged, as it is in the existing DP schemes~\cite{Wei2020Federated}.
	
	Apart from the specific global model aggregation $m$, the DP noise variance $\varTheta(m)$ depends on the privacy level $(\epsilon, \delta)$, the ratio of participating users $q$, the number of local update iterations between two global model aggregations $\tau$, and the iteration number $T$ (or, in other words, the number of global model aggregations $M$).
	It is of practical interest to determine the optimal values of $M$ and $\tau$, or their trade-off. 

	\subsection{Sensitivity and Privacy Analysis}\label{sec-privacy} 
	It is prudent to analyze the sensitivity and privacy performance of the proposed DP perturbation with time-varying noise variance $\varTheta(m)$. We use the $\ell_2$-norm to measure the sensitivity~\cite{dwork2014algorithmic}
	\begin{equation}\label{eq_sensitivity}
	\Delta s = \max_{{\cal D}, {\cal D}'} \left\| s({\cal D}) - s({\cal D}')\right\|,
	\end{equation}
	where $s(\cdot)$ is a general function in ${\cal D}$ (and ${\cal D}'$).	
	Accordingly, if the batch size for training the local model is consistent with the training sample number, the sensitivity is given by $\Delta s = \frac{2 C}{\left|{{\cal D}_k} \right|}$ with $C$ being the pre-determined clipping threshold and ${\cal D}_k$ being the local dataset at the $k$-th user~\cite{Wei2020Federated}.

	Given the sensitivity $\Delta s$, the amplitude (i.e., the standard deviation) of the noise injected into each global aggregation changes exponentially when $\vartheta \neq 1$. We derive the varying amplitude to meet the privacy requirement in \textbf{Theorem~\ref{theo_DP noise}}.
	\begin{theorem}\label{theo_DP noise}	
		To ensure the $(\epsilon,\delta)$-DP requirement of the local training dataset with $M$ global model aggregations, the amplitude of the DP noise in the first global model aggregation of the time-varying DP perturbation mechanism is given by
		\begin{equation}\label{eq_DP noise}	
		\sigma =  \frac{\Delta s }{\epsilon} \sqrt{2q\left( \frac{\vartheta - \vartheta^{1-M}}{\vartheta -1}\right)\ln\left(\frac{1}{\delta} \right) },\;{\text{if}}~\vartheta \neq 1.
		\end{equation}
	\end{theorem}
	
	\begin{proof}
		See Appendix~\ref{append-A}.
	\end{proof}

	\begin{remark}\label{rema-2}
		Theorem~\ref{theo_DP noise} indicates that a larger $\sigma$ results in a smaller $\epsilon$, i.e., a stronger privacy guarantee, and confirms that a larger $M$ leads to a higher likelihood of leaking private information during training, given $\sigma$.
		Based on Theorem~\ref{theo_DP noise}, given $\epsilon$, $\delta$, and $M$, we adjust the DP noise variance to balance privacy preservation and the convergence of FL training. 
	\end{remark}
	\begin{remark}\label{rema-3}
		Given a privacy budget $\epsilon$ for $M$ global aggregations, more clients involved in the model updates, i.e., a larger $q$ in \eqref{eq_DP noise}, lead to requirements of stronger perturbation noises being added to the local model of each involved client. This indicates less privacy leakage for each client, which is consistent with the conclusion drawn in~\cite{elkordy2022much}.
	\end{remark}
	\begin{remark}\label{rema-1}
		If $\vartheta = 1$, then $\sigma = \frac{\Delta s }{\epsilon}  \sqrt{2qM\ln\left(\frac{1}{\delta} \right) }$ does not change over the global model aggregations for $m=1,\cdots,M$, since $\frac{\vartheta - \vartheta^{1-M}}{\vartheta -1}~{\xrightarrow {\vartheta \to 1} }~M$. It is consistent with the amplitude of the DP noise in a Gaussian noise perturbation mechanism developed in~\cite{Wei2020Federated}.
	\end{remark}
	
	\subsection{Online Adjustment of DP Perturbation Noise Amplitude}\label{sec: noise update} 
	The aim of the proposed time-varying DP noise perturbation, i.e.,~\eqref{eq_DP noise}, is to protect the user privacy and ensure reasonable learning performance.
	As the global aggregation increases, however, the noise added in the late stage of the model training may degrade the learning performance. 
	
	To address this issue without compromising the $\left( \epsilon, \delta\right)$-DP privacy level, we can reduce the maximum number of global aggregations, $M$, and accordingly reduce the noise variance $\sigma_m^2$ at the $m$-th global aggregation. 
	Here, $\sigma'=\sigma'(M',m)$ depends on $M'$ and $m$, as given below.

	\begin{theorem}\label{theorem-noise update}
		To reduce the loss of learning at the $m$-th global aggregation without compromising the $\left( \epsilon, \delta\right)$-DP privacy of the learning, we update the maximum number of global aggregations to $M'$ ($M'<M$) and the variance of the perturbation noise to $\sigma_m^2=\vartheta^{m-1}(\sigma')^2$ with $\sigma'$ given by
		\begin{equation}\label{eq_DP noise re}
		\begin{aligned}
		\sigma' \! = \!\left\{ 
		\begin{aligned}
		&\!\frac{\Delta s}{\epsilon} \!\sqrt{2q\!\left(\!\frac{\vartheta - \vartheta^{1-m}}{\vartheta-1} + M'-m\!\right)\!\ln\!\left(\frac{1}{\delta} \right)\! },\,{{\rm if}~\vartheta > 1};\\
		&\!\frac{\Delta s}{\epsilon} \!\sqrt{2q M'\ln\left(\frac{1}{\delta} \right)},\,{{\rm if}~\vartheta = 1};\\
		&\!\frac{\Delta s}{\epsilon} \!\sqrt{2q\!\left(\!\frac{ \vartheta^{1-m} -\vartheta + \vartheta^{m-M'}}{1-\vartheta} \!\right)\!\ln\!\left(\frac{1}{\delta} \right)\! },\,{{\rm if}~\vartheta < 1}.
		\end{aligned}\right.
		\end{aligned}	
		\end{equation}
	\end{theorem}
	\begin{proof}
		See Appendix~\ref{appendix-theo-5}.
	\end{proof}

	Ideally, $M$ (or $M'$) should be as large as possible to improve the accuracy of FL training. In this sense, the online adjustment of the DP noise variance is of practical interest, as it allows for a progressive increase in $M'$ adapting to the convergence process of FL and the remaining privacy budget. This is attractive for applications that require a balance between accuracy and privacy.
	
	From Theorem \ref{theorem-noise update}, it is important to specify $M'$ and then~$\sigma'$ based on~\eqref{eq_DP noise re}. 
	In practice, the aggregator may have part of the dataset for testing purpose. If the test loss function value at the aggregator stops decreasing, then a new $M'$ is calculated.
	An empirical approach is to set $M'=\lceil\alpha_dM\rceil$ at the next global aggregation, whenever it is observed at a global aggregation that the global loss function stops decreasing. $\alpha_d$ can be empirically determined. $\lceil\cdot\rceil$ stands for ceiling.

	\section{Convergence of Privacy-Preserving FL under DP with Time-varying Noise Amplitude}\label{sec-convergence analysis}
	In this section, we establish the convergence upper bound for privacy-preserving FL (Algorithm~\ref{algo_DGD}) protected by the proposed DP mechanism with time-varying noise amplitudes. 
	
	\subsection{Definitions and Assumptions}
	We provide the following definitions and assumptions to facilitate analyzing the convergence of the FL under the new DP mechanism with time-varying perturbation noise variance.
	
	\begin{definition}[$\cal B$-local Dissimilarity]\label{def-dissimilarity}
		The local loss functions $f_k({\bm \omega}),\,k=1,\cdots, K$, yield $\cal B$-local dissimilarity at ${\bm \omega}$ if $\mathbb{E}_{{\cal D}_k} [\left\| \nabla f_k({\bm \omega})\right\|^2 ] \leq {\cal B}^2 \left\|  \nabla  F({\bm \omega}) \right\|^2$, where $\mathbb E_{{\cal D}_k}[\cdot]$ takes expectation over the distribution of user $k$'s dataset, ${\cal D}_k$, with the probability of ${\cal D}_k$ given by $p_k = {|{\cal D}_k|}/{\sum_{k\in {\cal K}}|{\cal D}_k|}$ and $\sum_{k \in {\cal K}} p_k =1$. If $\left\| \nabla F({\bm \omega}) \right\| \neq 0$, we define ${\cal B}({\bm \omega}) = \sqrt{{\mathbb{E}_k[\left\| \nabla f_k({\bm \omega})\right\|^2 ]}/{\left\| \nabla F({\bm \omega}) \right\|^2}}$.
	\end{definition}
	We refer to the gap in the gradient between the local and global loss functions as ``gradient divergence''. The gradients depend on the partition of data among the users.
	
	\begin{definition}[Gradient Divergence]\label{def-divergence}
		$\forall k $ and $\bm \omega$, $\gamma_k$ denotes an upper bound of the gradient divergence between the local and global loss functions, i.e.,
		$\left\|\nabla f_k({\bm \omega})-\nabla F({\bm \omega}) \right\| \leq \gamma_k$.
		The global gradient divergence is $\gamma \triangleq  \sum_{k} p_k \gamma_k = \frac{\sum_{k}|{\cal D}_k| \gamma_k}{|{\cal D}|}$.
	\end{definition}

	\begin{assumption}\label{assumption}
		$\forall k \in {\cal K}$, we make the following assumptions:
		\begin{enumerate}
			\item $f_k(\bm \omega)$ 
			is convex and
			$L$-smooth~\cite{o2006metric}, that is, $\left\|\nabla f_k({\bm \omega})-\nabla f_k({\bm \omega''}) \right\| \leq L \left\| {\bm \omega} - {\bm \omega''}\right\|,\,\forall {\bm \omega},{\bm \omega''} $, with $L$ being a constant depending on the loss function;
			
			\item $f_k(\bm \omega)$ is $L_c$-Lipschitz continuous, that is, $\left\| f_k({\bm \omega}) - f_k({\bm \omega''}) \right\| \leq L_c \left\| {\bm \omega} - {\bm \omega''}\right\|,\,\forall {\bm \omega},{\bm \omega''}$;

			\item The learning rate is $\eta \leq \frac{1}{L}$;
			\item $f_k(\bm \omega)$ fulfills the Polyak-Lojasiewicz requirement~\cite{karimi2016linear} with a positive parameter $\rho$, indicating that $F(\bm \omega)-F({\bm \omega}^*) \leq \frac{1}{2\rho} \left\| \nabla F(\bm \omega) \right\|^2 $ and ${\bm \omega}^*$ minimizes $F(\bm \omega)$;
			\item $F({\bm \omega}(0)) - F({\bm \omega}^*) 
			=  \Theta$, where $\Theta$ is a constant.
		\end{enumerate}
	\end{assumption}

	\subsection{Convergence Analysis}
	
	To analytically study the convergence of Algorithm~\ref{algo_DGD}, a corresponding centralized gradient descent-based learning process is typically considered, as given by 
	\begin{equation}
	{\bm \omega'}(t) = {{\bm \omega'}}(t-1)- \eta \nabla F({\bm \omega'}(t-1)).
	\end{equation}
	The model parameter ${\bm \omega'}(t)$ is updated using the global loss function $F(\cdot)$ and the entire dataset $\cal D$. 
	
	The following lemma~\cite[Thm. 1]{Wang2019Adaptive} provides an upper bound for the gap between the global model parameter of the proposed FL process, i.e.,  ${\bm \omega}(m)$, and the model parameter of the centralized learning process, i.e., ${\bm \omega'}(t)$, and, in turn, an upper bound of the gap between their loss functions, $F\left( {\bm \omega}(m) \right)$ and $F({\bm \omega'}(t))$. 
	$t = 1,\cdots,T$. 
	
	\begin{lemma}{\cite[Thm. 1]{Wang2019Adaptive}}\label{theo_inter}
		For any $ t=(m-1)\tau+1, \cdots,m \tau$, the difference of the global models between the FL process with DP perturbation and the centralized learning process in (9) is upper bounded; i.e., $\| {\bm \omega}(m) - {\bm \omega'}(t) \| \leq
		{\cal H}\left( t - (m-1)\tau \right)$. Here, ${\cal H}(x) \triangleq \frac{\gamma}{L}\left( (\eta L +1)^x -1\right) - \eta \gamma x$, $x = 0,1,2,\cdots$. $\gamma$ is the average gradient divergence over the dataset. 
		Since $F(\bm \omega)$ is $L_c$-Lipschitz, the difference of the global loss functions 
		is also upper bounded, i.e., $
		\|F\left( {\bm \omega}(m) \right) - F({\bm \omega'}(t))\| \leq L_c {\cal H}\left( t - (m-1)\tau \right) $.
	\end{lemma}
	
	By Lemma~\ref{theo_inter}, Definitions~\ref{def-dissimilarity} and~\ref{def-divergence}, and Assumption~\ref{assumption},
	we develop the following theorem to analyze the convergence bound of the gap between ${\bm \omega'}(t)$ and ${\bm \omega}^*$, $t\!=\!(m-1)\tau\!+\!1,\!\cdots\!,m\tau$. 
	\begin{theorem}\label{theo_convergence bound}
		To satisfy the $(\epsilon, \delta)$-DP, the convergence upper bound of the FL under time-varying DP noise perturbation, i.e., Algorithm~\ref{algo_DGD}, after $m$ global aggregation rounds, is obtained as
		\begin{equation}\label{eq_convergnece_bound 0}
		\begin{aligned}
		F({\bm \omega'}(t) )& - F({\bm \omega}^*) \leq {\cal A}^m \Theta +\\
		&  \frac{q L (\Delta s)^2\ln\left(\frac{1}{\delta} \right)\left(\vartheta^m - {\cal A}^m\right)\left( \vartheta-\vartheta^{1-m}\right) }{\epsilon^2 \left(\vartheta - {\cal A}\right) \left(U - 1 \right)},
		\end{aligned}
		\end{equation}	
		where ${\cal A} = 1+2\rho \phi $, and $\phi \triangleq \frac{\eta^2 L}{2}\left( \frac{(U-K){\cal B}^2}{K(U-1)} + \frac{K-1}{UK(U-1)} \right) - \eta$.
	\end{theorem}
	\begin{proof}
		See Appendix~\ref{appendix_convergence bound}.
	\end{proof}

	By setting $m = M$ in Theorems~\ref{theo_inter} and~\ref{theo_convergence bound}, we have the following corollary.
	
	\begin{corollary}\label{theo_convergence bound 2}
		To satisfy the $(\epsilon, \delta)$-DP, the convergence upper bound of the FL under time-varying DP noise perturbation, i.e., Algorithm~\ref{algo_DGD}, after all $M$ global aggregations (or, in other words, all $T$ iterations), is obtained as
		\begin{equation}\label{eq_convergnece_bound}
		\begin{aligned}
		& F\big({\bm \omega}({M}) \big) - F({\bm \omega}^*) \!\leq\! {\cal A}^{M} \Theta +\\
		& \;\! \frac{qL (\Delta s)^2\ln\left(\frac{1}{\delta} \right)\left(\vartheta^{M} - {\cal A}^{M}\right)\left( \vartheta-\vartheta^{1-{M}}\right) }{\epsilon^2\left(\vartheta - {\cal A}\right) \left(U - 1 \right)} \!+\! L_c {\cal H}\left(\tau\right).
		\end{aligned}
		\end{equation}	
	\end{corollary}
	
	Corollary~\ref{theo_convergence bound 2} shows the trade-off between the learning performance and privacy protection level. In the case of a weak privacy guarantee (i.e., both $\epsilon$ and $\delta$ are large), the convergence upper bound is tighter since the second term on the right-hand side (RHS) of~\eqref{eq_convergnece_bound} approaches zero. 
	Additionally, the convergence upper bound in~\eqref{eq_convergnece_bound} is unrestricted by the partition of the data among the users. The data partition is captured in the gradient divergence $\gamma$, a parameter of ${\cal H}\left(\tau \right)$; see Lemma~\ref{theo_inter}. 
	Since $\frac{\partial {\cal H}\left(\tau \right)}{\partial M}=\frac{\partial {\cal H}\left(\frac{T}{M}\right)}{\partial M} = -\frac{\gamma T}{M^2}\left(\eta L +1\right)^{\frac{T}{M}} \ln\left(\eta L +1\right)  + \frac{\gamma T}{M^2} \leq 0 $ for $M \leq \frac{T}{\ln\left(\eta/\ln\left(\eta L + 1 \right) \right)}$, the upper bound becomes smaller when $M$ is larger and $\gamma$ is smaller under a given $T$. 
	
	Another interesting finding in Theorem~\ref{theo_convergence bound} is that the convergence upper bound is \emph{not} a monotonic function of either $T$ or $M$, as established in the following Corollary~\ref{rema}.
	\begin{corollary}\label{rema}
		Given $\epsilon$ and $\delta$, the convergence upper bound is convex in the number of iterations or global aggregations, i.e., $T$ or $M$, if $\vartheta \geq {\cal A}$ and $\tau \geq \frac{\ln\left(\eta/\ln\left(\eta L + 1 \right) \right)}{\ln\left(\eta L + 1 \right)}$. 	
	\end{corollary}
	\begin{proof}
		See Appendix~\ref{append_coro}.
	\end{proof}

	Note that the upper bound in \eqref{eq_convergnece_bound} may not be tight, as it is obtained through the use of the triangular inequality and Jensen's inequality, which are commonly used in the derivation of convergence upper bounds for FL models \cite{wu2021incentivizing,geyer2017differentially, lu2019differentially,bhowmick2018protection,Ma2020On,bassily2014private}. 
	Despite this, the  upper bound reveals that the optimality gap of an FL model under our proposed perturbation mechanism would decrease at first, and then increase as the number of global aggregations increases. This confirms that  FL models can diverge, rather than converge, due to noise perturbation for differential privacy, and highlights the existence of an optimal number of global aggregations; i.e., 
	$T$ or $M$ can be optimized to minimize the loss function. 
	
	Also note that \textbf{Theorem \ref{theo_convergence bound}}, and \textbf{Corollaries \ref{theo_convergence bound 2}} and \textbf{\ref{rema}} are based on the smoothness of the loss functions of the neural network models being trained, and applicable to the convergence bound analysis of FL systems training MLP and SVM models.
	On the other hand, \textbf{Theorems \ref{theo_DP noise}} and \textbf{\ref{theorem-noise update}} specify the DP noise variances to preserve the privacy of an FL training process and defer the divergence of the process caused by the DP noises. The two theorems do not rely on the smoothness of the loss functions, and are applicable to neural networks with non-smooth loss functions, e.g., CNN.  
	
	\subsection{Optimal Global Aggregation Rounds $M^*$}\label{subsec-optimal-M}
	To improve the convergence of Algorithm~\ref{algo_DGD}, we optimize $M$ to minimize the upper bound of $F({\bm \omega}(M)) - F(\bm{\omega}^*)$ for a given $T$. 
	By choosing a small enough $\eta \leq \frac{1}{L}$, Problem \eqref{eq-problem} can be rewritten as 
	\begin{equation}\label{eq_op_upper}
	\begin{aligned}
	\underset{M}{\min} & \,{\cal A}^{M} \Theta + \frac{qL(\Delta s)^2\ln\left(\frac{1}{\delta} \right)\left(\vartheta^{M} - {\cal A}^{M}\right)\left( \vartheta-\vartheta^{1-{M}}\right) }{\epsilon^2\left(\vartheta - {\cal A}\right) \left(U - 1 \right)}. \\
	\end{aligned}	
	\end{equation}
	which can be solved by setting the first-order derivative of its objective, denoted by ${\cal G}(M)$, to zero, i.e.,
	\begin{equation}\label{eq: First-order derivative}
	\begin{aligned}
	\frac{\partial {\cal G}(M)}{\partial M} = &{\Theta} {\cal A}^{M} \ln\left({\cal A}\right) + 
	\frac{qL(\Delta s)^2\ln\left(\frac{1}{\delta}\right)\vartheta}{\epsilon^2 \left(\vartheta-{\cal A} \right)\left(U-1\right)}\\
	&  \big(\vartheta^{M} \ln\left(\vartheta\right) -{\cal A}^{M} \ln\left({\cal A}\right)+\left({{\cal A}}/{\vartheta}\right)^{M}\\
	&\ln\left({{\cal A}}/{\vartheta}\right) \big)
	+ L_c \frac{\partial {\cal H}\left(\frac{T}{M}\right)}{\partial M}= 0.  
	\end{aligned}
	\end{equation}
	The optimal number of global aggregations, denoted by $M^*$, can be found numerically~\cite{boyd2004convex}, e.g., using bisection search. 
	According to Corollary~\ref{rema}, the convergence upper bound is convex  with respect to $M$, if $\vartheta \geq {\cal A}$ and $\tau \geq \frac{\ln\left(\eta/\ln\left(\eta L + 1 \right) \right)}{\ln\left(\eta L + 1 \right)}$. Then, the solution to \eqref{eq: First-order derivative} is unique and globally optimal. 
	The value of $\tau$ is determined by the ratio of $T$ to $M$.
	Here, given $T$, we determine $M$. Accordingly, $\tau=T/M$ is updated.
	
	\section{Experimental Results}\label{sec-results}
	This section assesses the accuracy of our analysis and the effectiveness of the proposed time-varying DP noise variance for distributed gradient descent-based FL under various learning tasks, models, and real-world datasets.
	
	\subsection{Experimental Settings}\label{subsec-exp-set}
	We set the number of users to $U = 100$, the number of chosen users to $K = 10$, and the number of iterations between two consecutive global aggregations to $\tau = 5$. With reference to~\cite{Wei2020Federated}, we set the clipping threshold $C = 5$, and $\delta = 0.001$. The maximum number of global aggregations is $M =30$ and the privacy protection level is $\epsilon =10$; unless otherwise specified.
	The experiments are conducted on four datasets:
	\begin{itemize}
		\item The standard MNIST dataset comprises 60,000 training and 10,000 testing examples, which are grayscale images of handwritten digits from one to ten; 
		
		\item The ADULT dataset, which contains 40,000 records extracted from census data~\cite{kennedy1995particle} and each record has up to 58 attributes, including age, education, etc.; 
		
		\item The CIFAR10 dataset, which contains 60,000  $32 \times 32$ color images in ten classes (6,000 per class), 50,000 for training and 10,000 for testing; and
		
		\item The Fashion-MNIST (FMNIST) dataset, which contains Zalando's article images (i.e., $28 \times 28$ grayscale images) in ten classes, including 60,000 examples for training and 10,000 examples for testing.  
	\end{itemize}
	We evaluate the proposed DP mechanism with time-varying perturbation noise variance  on MLP, SVM, and CNN models.

	\subsection{Evaluation of the Distributed Gradient Descent Time-Varying DP Algorithm on MLP}\label{experiment: MLP}

	The MLP is a fully-connected feedforward neural network. We consider an MLP model comprising a hidden layer and 32 hidden units, and train the model on the MNIST dataset. We adopt linear activation functions and softmax of ten classes that correspond to the ten digits. The model's error on the local training dataset is measured by a cross-entropy loss function.

	\subsubsection{Impact of Noise Scaling Factor $\vartheta$}\label{subsec-noise}
	\begin{figure}[!t]
		\centering
		\setlength{\abovecaptionskip}{0.1cm}
		\begin{subfigure}[b]{0.24\textwidth}
			\centering
			\includegraphics[width=\textwidth]{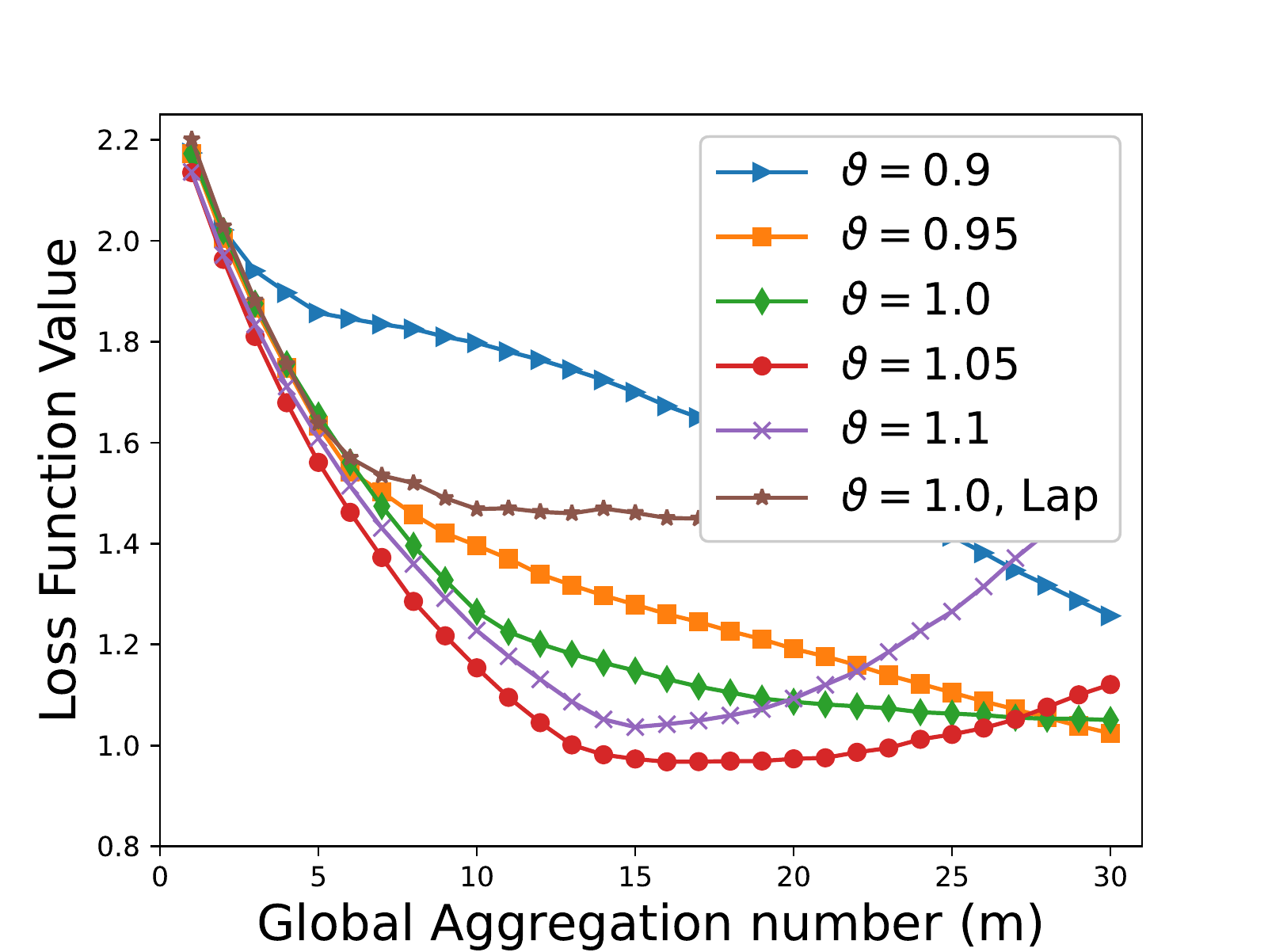}
			\caption{Loss function value vs. $m$}
			\sublabel{fig-loss-noise-m-ep5}
		\end{subfigure}
		\begin{subfigure}[b]{0.24\textwidth}
			\centering
			\includegraphics[width=\textwidth]{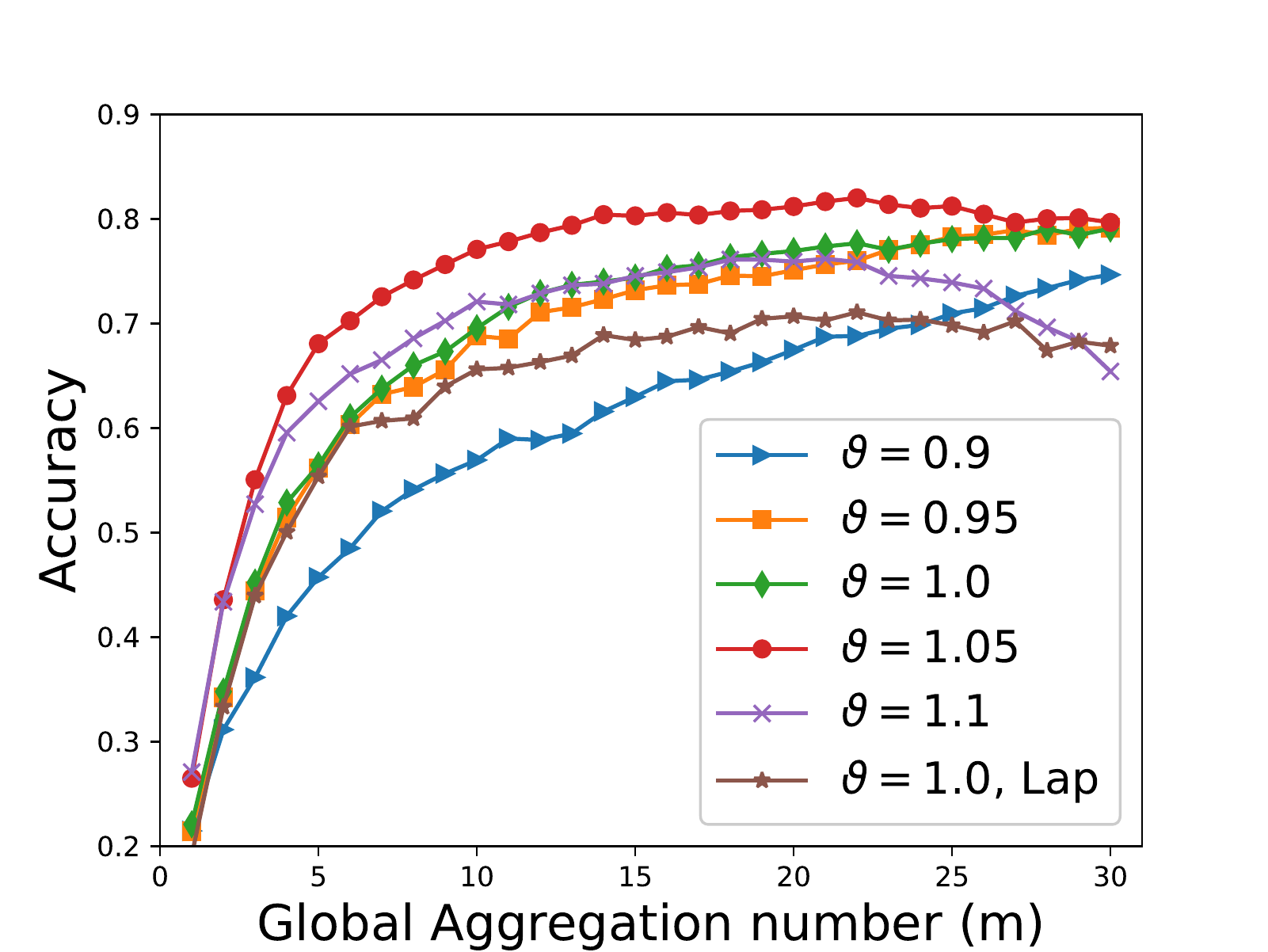}
			\caption{Accuracy vs. $m$}
			\sublabel{fig-acc-noise-m-ep5}
		\end{subfigure}
		\caption{Comparison of loss function value and accuracy of the MLP model trained on the MNIST dataset with increasing global aggregation rounds $m$ under different values of $\vartheta$ (`Lap' stands for the Laplacian mechanism).}
		\label{fig-noise-mlp-m-ep5}
	\end{figure} 
	\begin{figure}[!t]
		\centering
		\begin{subfigure}[b]{0.24\textwidth}
			\centering
			\includegraphics[width=\textwidth]{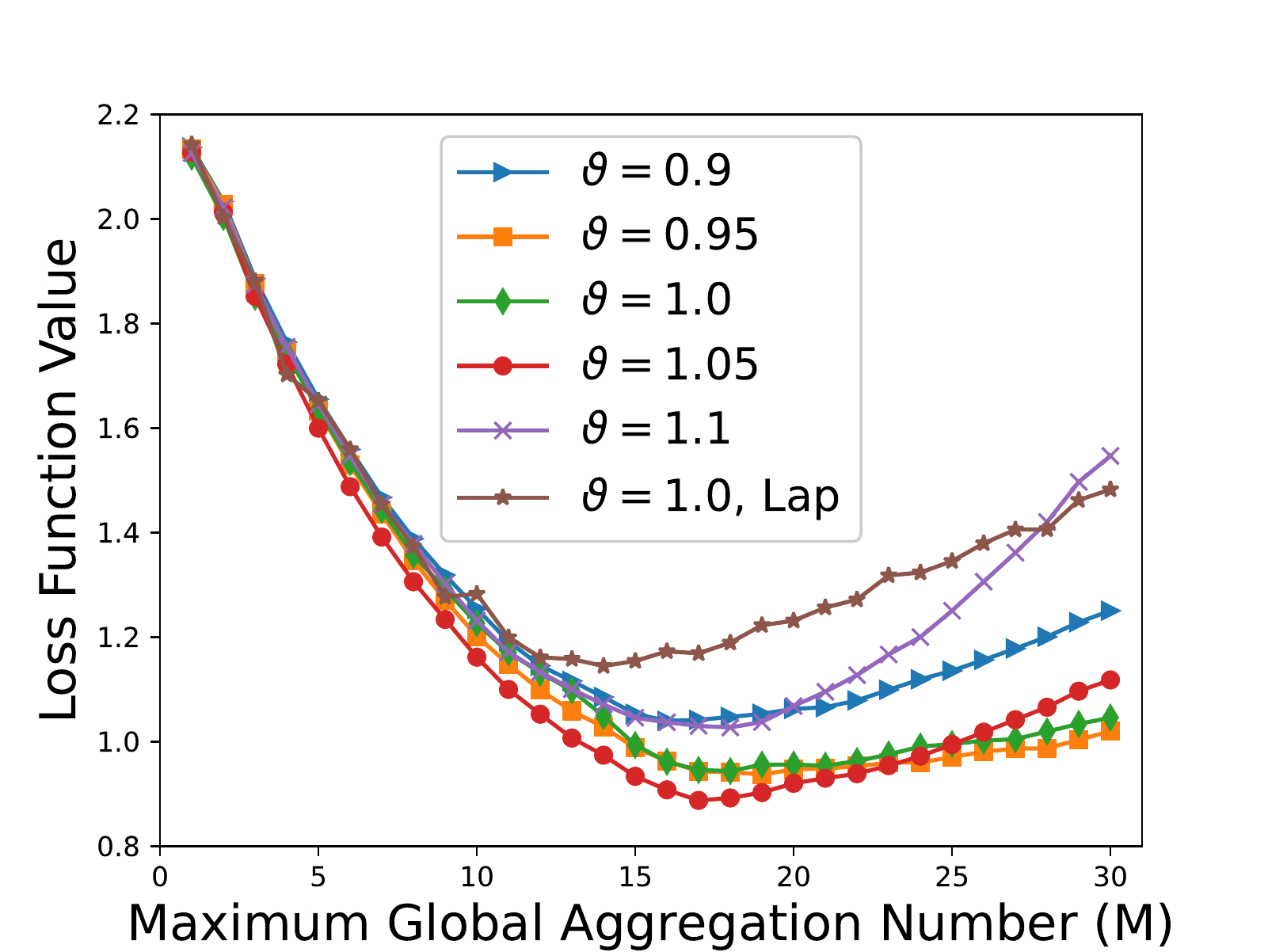}
			\caption{Loss function value vs. $M$}
			\sublabel{fig-loss-noise-M-lep5}
		\end{subfigure}
		\begin{subfigure}[b]{0.24\textwidth}
			\centering
			\includegraphics[width=\textwidth]{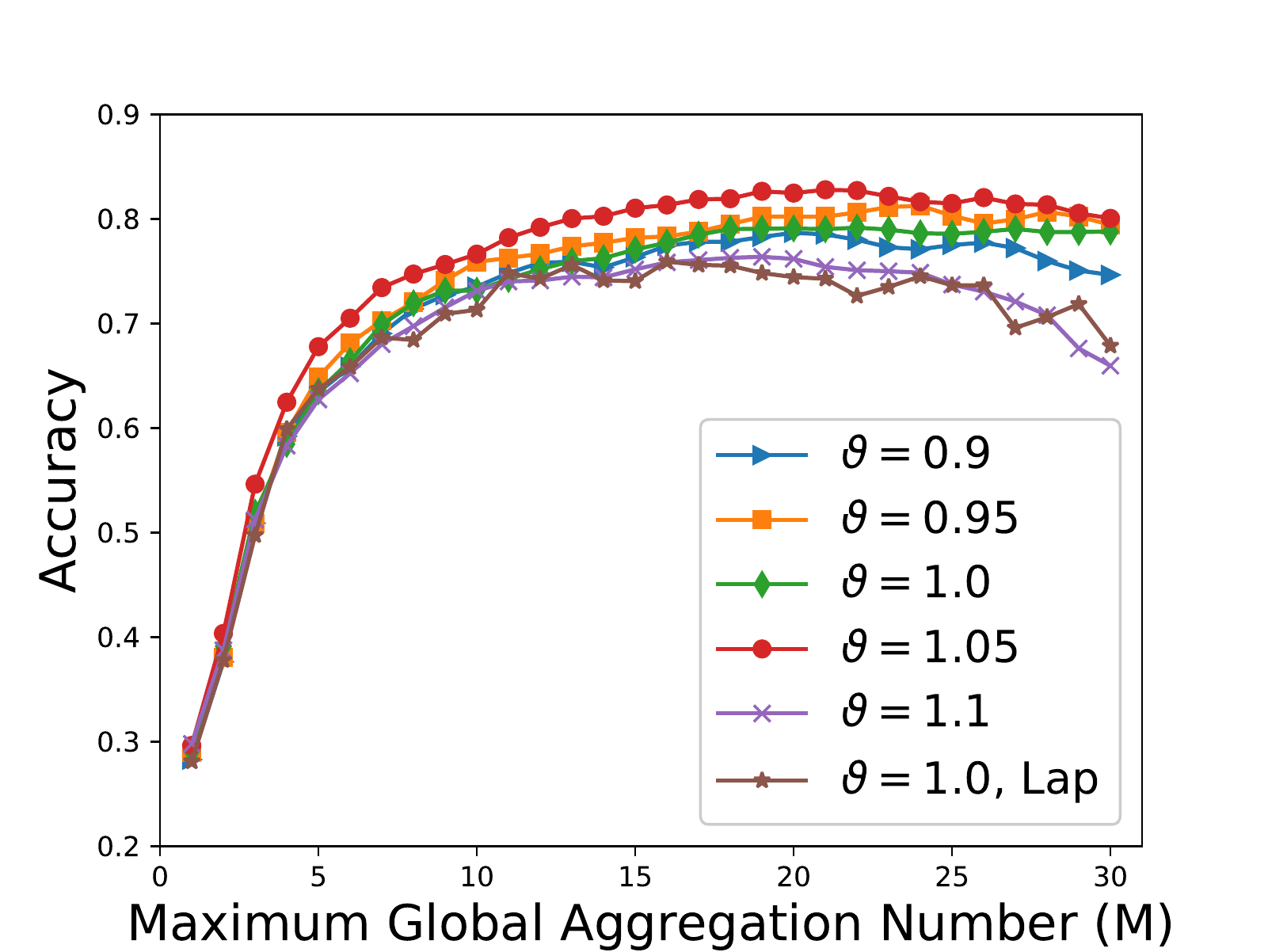}
			\caption{Accuracy vs. $M$}
			\sublabel{fig-acc-noise-M-lep5}
		\end{subfigure}
		\caption{Comparison of loss function value and accuracy of the MLP model on the MNIST dataset with respect to the maximum number of global aggregations $M$ under different values of $\vartheta$ with a fixed $\epsilon = 10$ (`Lap' stands for the Laplacian mechanism).}
		\label{fig-noise-mlp-M-lep5}
	\end{figure}

	Fig.~\ref{fig-noise-mlp-m-ep5} evaluates the impact of the proposed DP mechanism with time-varying perturbation noise variance on the convergence of FL, where the $x$-axis provides the index to the global aggregations $m$.
	Here, $m \leq M=30$, since using a value of $m$ larger than $M$ would violate the privacy requirements and render the values of the training accuracy meaningless.
	Different $\vartheta$ values are assessed, where  
	$\vartheta=1$ can be viewed as the state-of-the-art design of the DP perturbation noise developed in~\cite{Wei2020Federated}.
	Each curve in a figure corresponds to a standalone experiment of a training process given $\vartheta$ and $\epsilon$.
	
	Fig.~\ref{fig-loss-noise-m-ep5} shows that given the privacy level, the number of global aggregations needed for the (testing) loss function value to reach its minimum generally declines, as $\vartheta$ rises from 0.9 to 1.1.  
	In the case of $\vartheta > 1$, the minimum of the loss function 
	first declines and then increases.  
	This is because the noise added in the early learning stage of the learning is smaller in the case of $\vartheta > 1$ than it is in the case of $\vartheta = 1$, leading to faster convergence. 
	Moreover, the noise rises exponentially and leads to an increase in the loss function value with the growth of $m$.  
	In the case of $\vartheta < 1$, the loss function value is large in the early learning stage and declines with the growth of $m$. The loss function value also decreases, as $\vartheta$ increases from $0.9$ to $0.95$. This is because the noise added in the early stage is larger for a smaller $\vartheta$, causing slower convergence. 
	
	Fig.~\ref{fig-acc-noise-m-ep5} plots the (testing) accuracy of the proposed algorithm under different $\vartheta$ values. 
	Similar to the loss function value, given the privacy level, the aggregation numbers required to reach the maximum accuracy decrease as $\vartheta$ increases from 0.9 to 1.1.
	The accuracy is more stable and remains unchanged under a smaller $\vartheta$ value. In contrast, the accuracy quickly reaches its peak and declines slowly under a larger $\vartheta$ value since a smaller noise is added in the early learning stages of the model training process and leads to faster convergence given $\epsilon$. With the increase of global aggregations, the noises added to the model parameters grow exponentially, resulting in accuracy degradation. To this end, the value of $\vartheta$ can be adequately configured to achieve better convergence and satisfy the DP requirement in the FL process with the proposed time-varying DP perturbation noise variances.
	
	\begin{table}[!t]
		\centering
		\renewcommand\tabcolsep{2.0pt}
		\renewcommand{\arraystretch}{1.5}
		\caption{The initial amplitude of the DP noise and minimum loss function value under different values of $\vartheta$. The privacy level is $\epsilon =10$.}
		\begin{tabular}{c|c|c|c|c|c}
			\toprule[1.5pt]
			\specialrule{0em}{2pt}{2pt}
			$\vartheta$  & 0.9 & 0.95  & 1.0   & 1.05   & 1.1   \\ \hline \specialrule{0em}{2pt}{2pt}
			$\sigma$     & 0.003556 & 0.003654  & 0.003749 & 0.003841 & 0.003932 \\ \hline
			Min. loss func.  &1.03783 & 0.92706  & 0.94142  & 0.88862 & 1.01871  \\
			\bottomrule[1.5pt]
		\end{tabular}
		\label{tab-noise-std}
	\end{table}

	Fig.~\ref{fig-noise-mlp-M-lep5} evaluates the impact of the proposed DP mechanism with time-varying perturbation noise variance on the utility (i.e., the loss function and accuracy) of FL, as the total number of global aggregations, $M$, increases. 
	Each point in a figure corresponds to a standalone experiment of a training process given $\vartheta$ and $\epsilon$.
	The initial amplitude of the DP noise is calibrated for each point based on the given value of $\epsilon$; see the second row of Tab.~\ref{tab-noise-std}.
	
	Fig.~\ref{fig-loss-noise-M-lep5} shows that the loss function of the MLP exhibits convex curvature with respect to $M$, as is consistent with Corollary~\ref{rema}.
	The loss functions achieve their minimums under the optimal number of global aggregations, $M^* =17$, for all considered $\vartheta$ values, validating the result in Section~\ref{subsec-optimal-M}.
	The third row of Tab.~\ref{tab-noise-std} provides the corresponding minimum loss function values.
	Fig.~\ref{fig-loss-noise-M-lep5} and Tab.~\ref{tab-noise-std} reveal that both the time-increasing noise perturbation ($\vartheta = 1.05$) and the time-decreasing noise perturbation ($\vartheta = 0.95$) can outperform the time-invariant noise perturbation ($\vartheta = 1.0$) in terms of loss at their respective optimal numbers of global aggregations.

	Fig.~\ref{fig-acc-noise-M-lep5} plots the (testing) accuracy of the proposed algorithm against $M$, where different values of $\vartheta$ are considered. 
	Consistent with Fig.~\ref{fig-loss-noise-M-lep5}, the 
	optimal $M^*$ values achieve the best accuracy 
	in all experiments.
	We also see that when $\vartheta\neq 1$, a large number of global aggregations $M>M^*$ could overkill the learning accuracy, especially when $\vartheta$ is big, e.g., $\vartheta =1.1$.
	One can potentially select the optimal $\vartheta$ value to achieve the optimal learning accuracy.
	In the example of Fig.~\ref{fig-noise-mlp-M-lep5}, the configuration of $\vartheta \geq 1.05$ allows for the smallest loss and highest accuracy of the learning.
	As observed in Fig.~\ref{fig-noise-mlp-M-lep5}, all curves diverge, i.e., convex in Fig.~\ref{fig-noise-mlp-M-lep5}(a) and concave in Fig.~\ref{fig-noise-mlp-M-lep5}(b), even when $\vartheta =1.0$. This is because the local models are still perturbed by the DP noises with a persistent DP noise variance, when $\vartheta =1.0$.
	
	Figs.~\ref{fig-noise-mlp-m-ep5} and \ref{fig-noise-mlp-M-lep5} also compare our proposed DP mechanism with time-varying perturbation noise variance to the Laplacian mechanism (with a constant noise variance)~\cite{fu2021practicality}. Except for the DP mechanisms, all curves are under consistent experimental parameter settings in the figures.  It is observed that the FL performance is worse under the Laplacian mechanism than it is under our proposed Gaussian mechanism with a time-varying DP noise variance. This is because the Laplacian mechanism adds larger noises to the data by sampling from a Laplacian distribution, which has a higher likelihood of sampling values that are farther away from the mean.
	
	\subsubsection{Impact of Privacy Protection Level $\epsilon$}
	\begin{figure}[!t]
		\centering
		\begin{subfigure}[b]{0.24\textwidth}
			\centering
			\includegraphics[width=\textwidth]{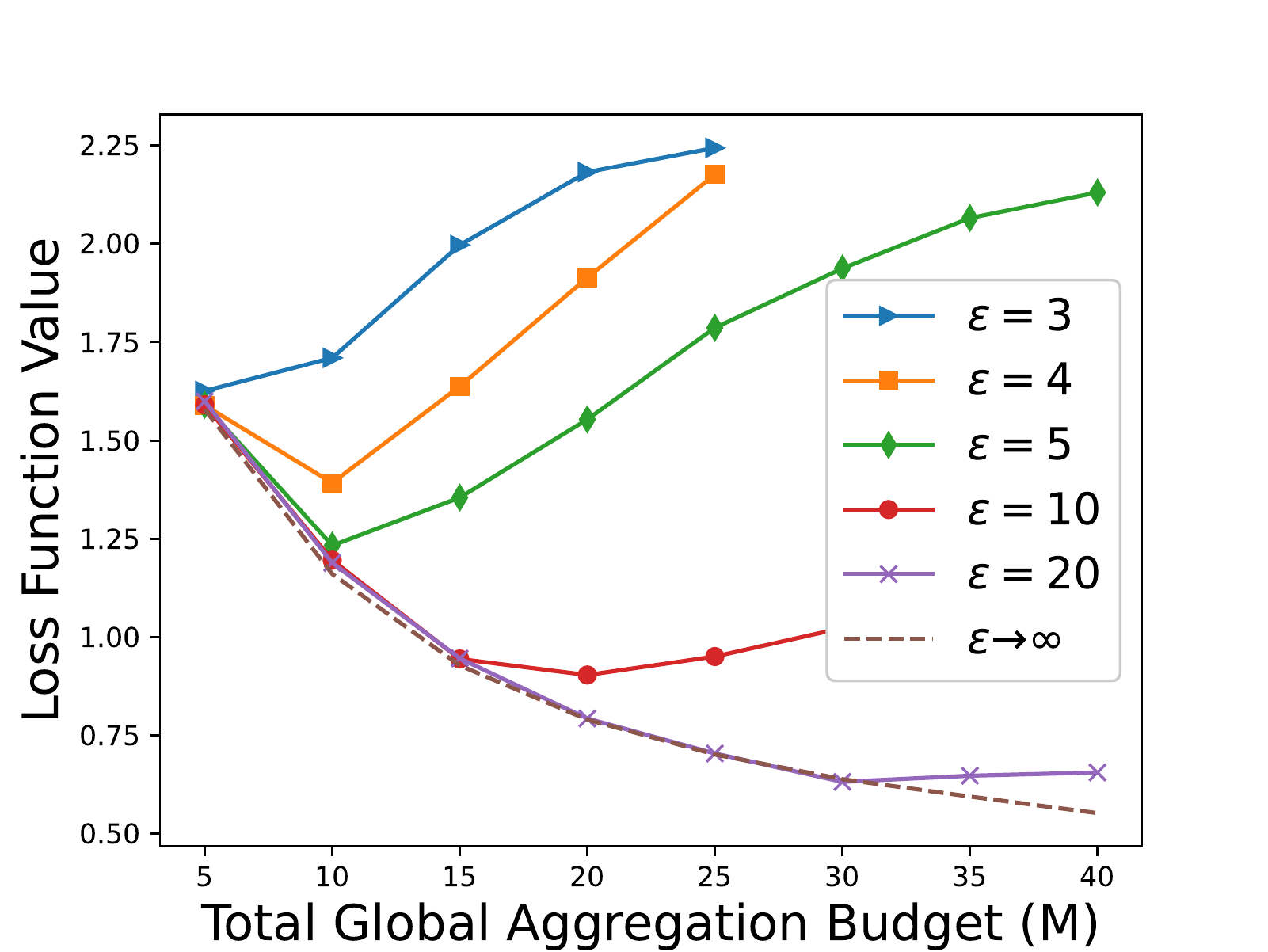}
			\caption{Loss function value vs. $M$}
			\sublabel{fig-loss-privacy-ep10}
		\end{subfigure}
		\hfill
		\begin{subfigure}[b]{0.24\textwidth}
			\centering
			\includegraphics[width=\textwidth]{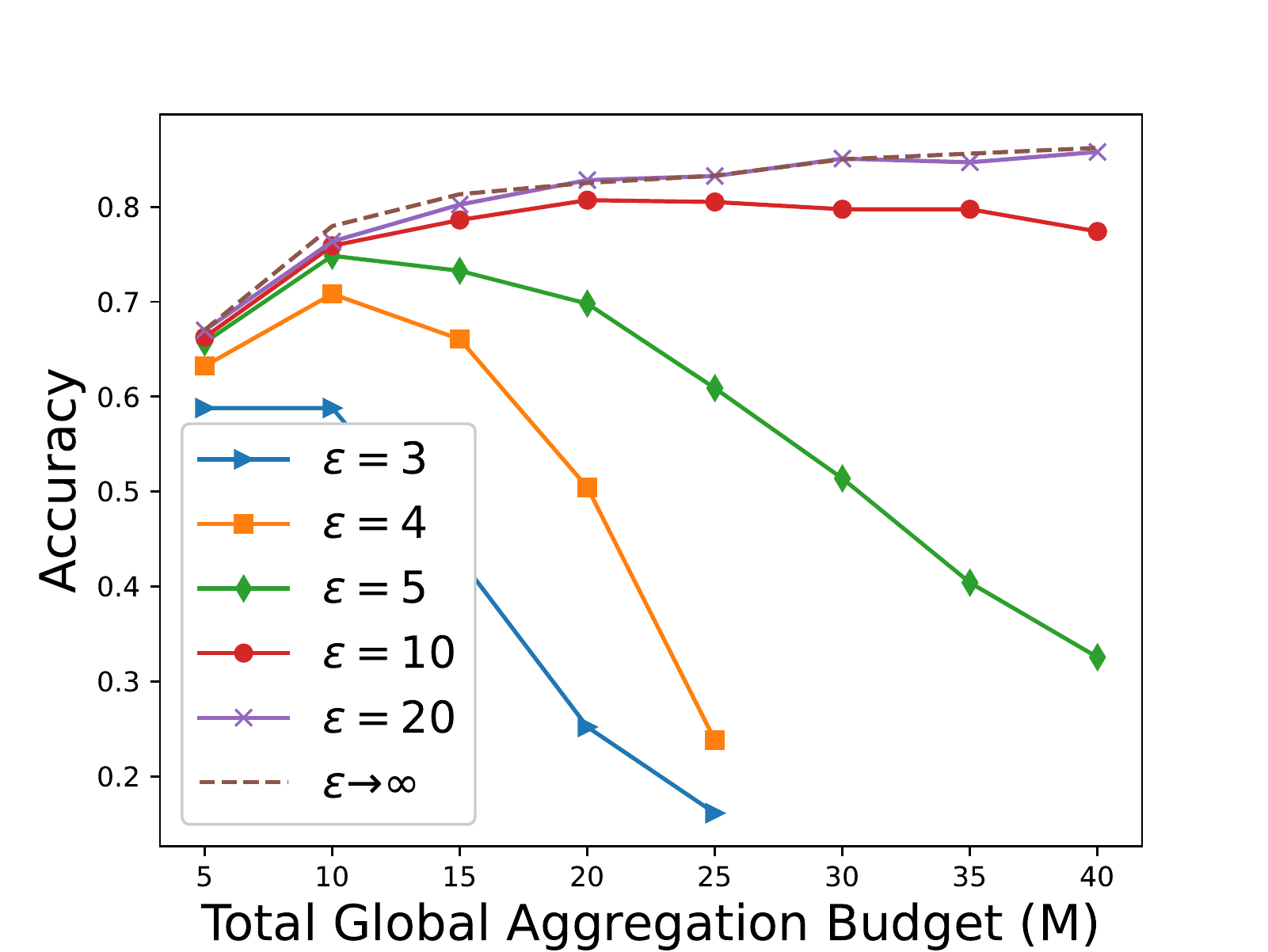}
			\caption{Accuracy vs. $M$}
			\sublabel{fig-acc-privacy-ep10}
		\end{subfigure}
		\caption{Loss function value and Accuracy of the MLP model on the MNIST dataset vs. the maximum number of global aggregations $M$ under different values of $\epsilon$, where $\vartheta = 1.05$.}
		\label{fig-privacy-mlp}
	\end{figure}

	We evaluate the impact of $\epsilon$ on the utility of the FL protected by the proposed DP mechanism with time-varying perturbation noise variance. According to Section~\ref{subsec-noise}, we set $\vartheta = 1.05$. 
	Fig.~\ref{fig-privacy-mlp} evaluates the loss  of the learning with the growth of $M$ under different settings of the privacy level $\epsilon$.
	For comparison, we also plot the case  with no DP noise perturbation, i.e., $\epsilon \to \infty$.
	Figs.~\ref{fig-loss-privacy-ep10} and~\ref{fig-acc-privacy-ep10} show that the loss function value is convex and the accuracy is concave with respect to $M$, which is in line with Corollary~\ref{rema}. 
	In general, the optimal number of global aggregations $M^*$ increases with~$\epsilon$.
	
	Fig.~\ref{fig-loss-privacy-ep10} also shows that the loss function values decrease and approach the case with no DP noise perturbation, as $\epsilon$ increases. 
	Fig.~\ref{fig-acc-privacy-ep10} shows that increasing $\epsilon$ can improve the accuracy. 
	In the case that the privacy level is larger than 20, i.e., $\epsilon \geq 20$, the convergence performance 
	approaches the case without noise perturbation. 
	This is because, with a lower privacy protection level $\epsilon$, the DP noises with a larger variance are injected in the first global aggregation according to Theorem~\ref{theo_DP noise}, resulting in larger loss function values. 
	A trade-off arises between the learning performance (i.e., loss and accuracy) and privacy level, and can be adjusted through~$M^*$.

	\subsubsection{Online Adjustment of DP Noise Variance}
	\begin{figure}[!t]
		\centering
		\includegraphics[width=1\columnwidth]{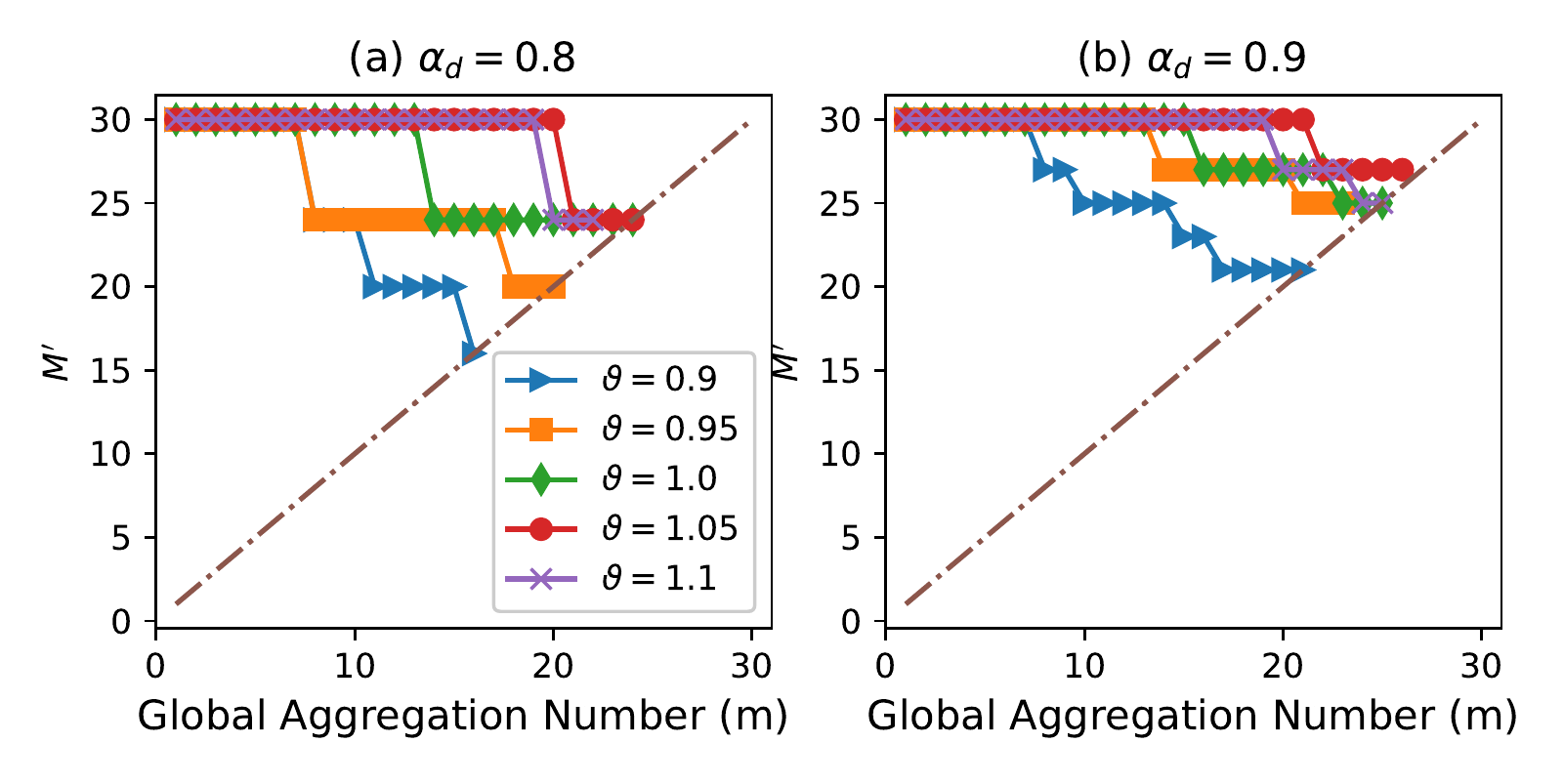}
		\caption{The updated maximum number of global aggregations $M'$ with the increasing number of global aggregations $m$, where $M'$ and $\sigma_m$ are updated according to Theorem~\ref{theorem-noise update} whenever the global loss function stops decreases.}
		\label{fig-M-update}
	\end{figure}
	\begin{figure}[!t]
		\centering
		\includegraphics[width=1\columnwidth]{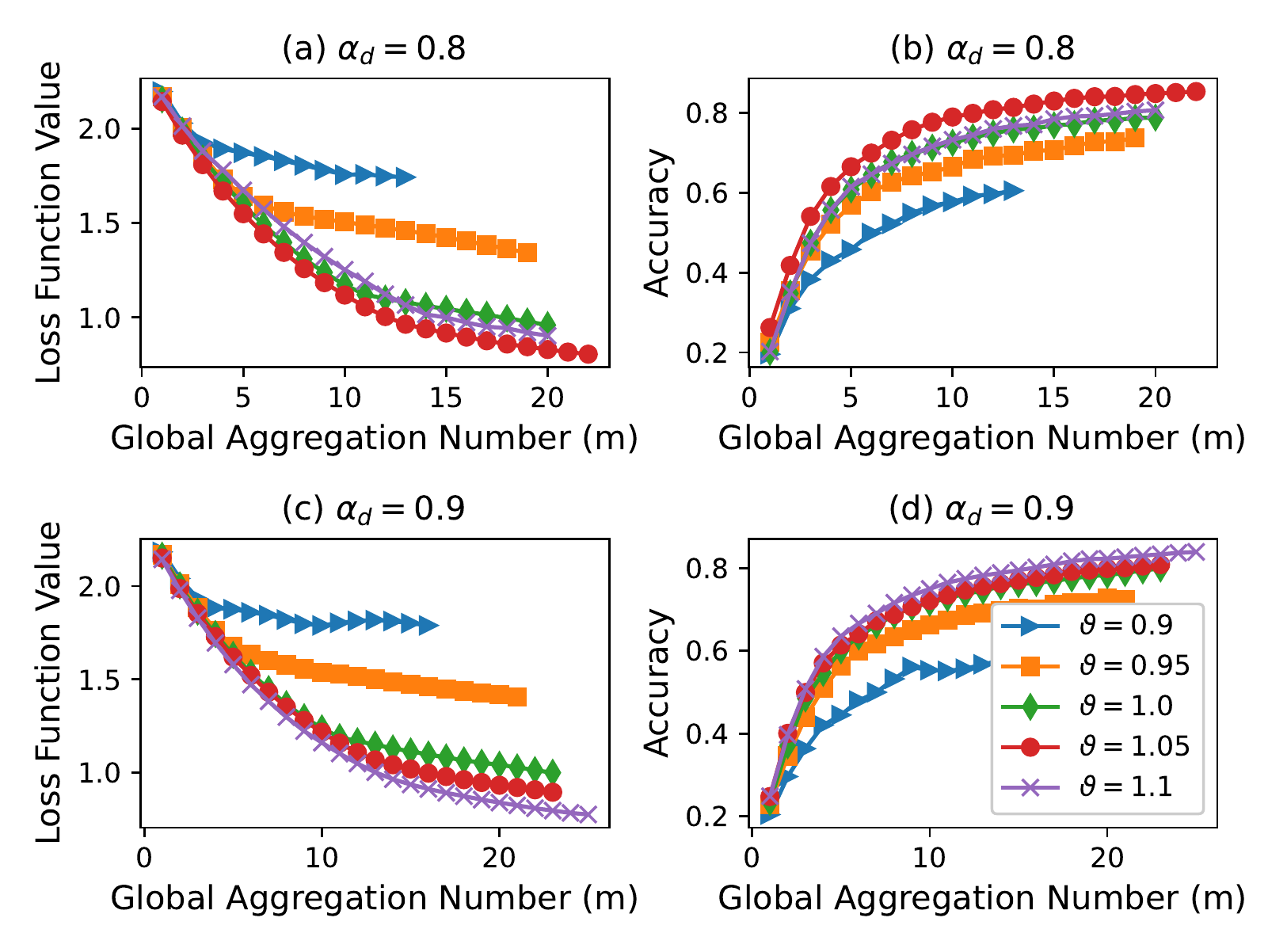}
		\caption{Loss function value and Accuracy of the MLP model on the MNIST dataset vs. $m$, where  $\vartheta$ takes different values. The curves of $\vartheta=0.9$ terminate earlier at smaller $m$ values than the other curves, because they can accommodate smaller numbers of global aggregations given the privacy level~$\epsilon$.}
		\label{fig-noise-update}
	\end{figure}

	We evaluate the online adjustment of the aggregation number and DP noise variance proposed in Section~\ref{sec: noise update}.
	Figs.~\ref{fig-M-update} and~\ref{fig-noise-update} plot the updated maximum number of global aggregations, i.e., $M'$, the loss function value, and the accuracy under different values of $\vartheta$ when $\alpha_d= 0.8$ and  $0.9$.
	As shown in Fig.~\ref{fig-M-update}, the maximum number of global aggregations $M'$ decreases over $m$, resulting from the updating of $\sigma_m$ to keep the global loss function decreasing, and meanwhile, the privacy protection level $\epsilon$ satisfied. 
	As shown in Figs.~\ref{fig-noise-update}(a) and~\ref{fig-noise-update}(c),  the loss function value decreases with the increase of $M$ under all considered $\vartheta$ values. 
	Among all four curves, $\vartheta = 1.05$ provides the best learning performance when $\alpha_d = 0.8$, as shown in Figs.~\ref{fig-noise-update}(a) and \ref{fig-noise-update}(b), and $\vartheta = 1.1$ provides the best learning performance when $\alpha_d = 0.9$, as shown in Figs.~\ref{fig-noise-update}(c) and~\ref{fig-noise-update}(d). 
	Compared to the results without online adjustment of the DP noise variance, i.e., Figs.~\ref{fig-loss-noise-M-lep5} and~\ref{fig-acc-noise-M-lep5}, the online adjustment of the DP noise variance can improve the learning performance of both the loss and accuracy in the case of $\vartheta \geq 1$; see Fig.~\ref{fig-noise-update}. 
	
	\subsubsection{Defence against Membership Inference Attacks}
	\begin{table}[!t]
		\caption{Comparison of membership inference attack success rate with and without the proposed DP mechanism, averaged over 50 independent trials with 25 shadow models.}
		\centering
		\begin{tabular}{|c|c|c|c|c|}
			\hline
			& $\vartheta = 0.95$ & $ \vartheta = 1.0$ & $\vartheta = 1.05$ & no DP                   \\ \hline
			$\epsilon = 5$   & 0.581             & 0.582             & 0.573  & \multirow{3}{*}{0.993} \\ \cline{1-4}
			$\epsilon = 10$ & 0.599             & 0.587             & 0.583 &                         \\ \cline{1-4}
			$\epsilon = 20$ & 0.607             & 0.602             & 0.596              &    
			\\ \hline
		\end{tabular}\label{tab-mia}
	\end{table}
	
	Table~\ref{tab-mia} evaluates the effectiveness of the proposed method in defending against membership inference attacks, which aim to determine whether a specific data point was used in training the model. The results indicate that without the mechanism, the attack success rate of the membership inference attack can be as high as 99.3\%. By contrast, when the mechanism is employed, the attack success rate drops significantly, e.g., by about 40\% to less than 60\% when $\vartheta = 1.05$.

	\subsection{Extension to SVM and CNN models}
	The proposed DP mechanism with time-varying perturbation noise can be readily applied to SVM and CNN models:
	\begin{itemize}
		\item The SVM model is trained using a standard quadratic optimization algorithm on the ADULT dataset. The loss function is $F({\bm \omega}) = \frac{\lambda_r}{2}\left\|{\bm \omega} \right\|_2^2 + \max\left\lbrace0,\; \beta_n - {\bm \omega}^T {\cal D}_{k,n} \right\rbrace$, where $\lambda_r > 0$ is a regularization coefficient; ${\cal D}_{k,n}$ is the $n$-th sample in ${\cal D}_{k}$, i.e.,  the dataset at the $k$-th user; and $\beta_n \in \{-1,1\}$ for $n =1, \cdots, \left| {\cal D}_{k}\right|$. 
		
		\item The CNN model contains two convolutional layers with a kernel size of five and three fully-connected layers. The CNN model is trained separately on the CIFAR10 and FMNIST datasets. We adopt the ReLU units and softmax of ten classes for the ten classes of the CIFAR10 and the ten digits of the FMNIST. The CNN model is trained using the SGD to minimize the loss function. 
	\end{itemize}
	
	\begin{figure}[!t]
		\centering
		\begin{subfigure}[b]{0.24\textwidth}
			\centering
			\includegraphics[width=\textwidth]{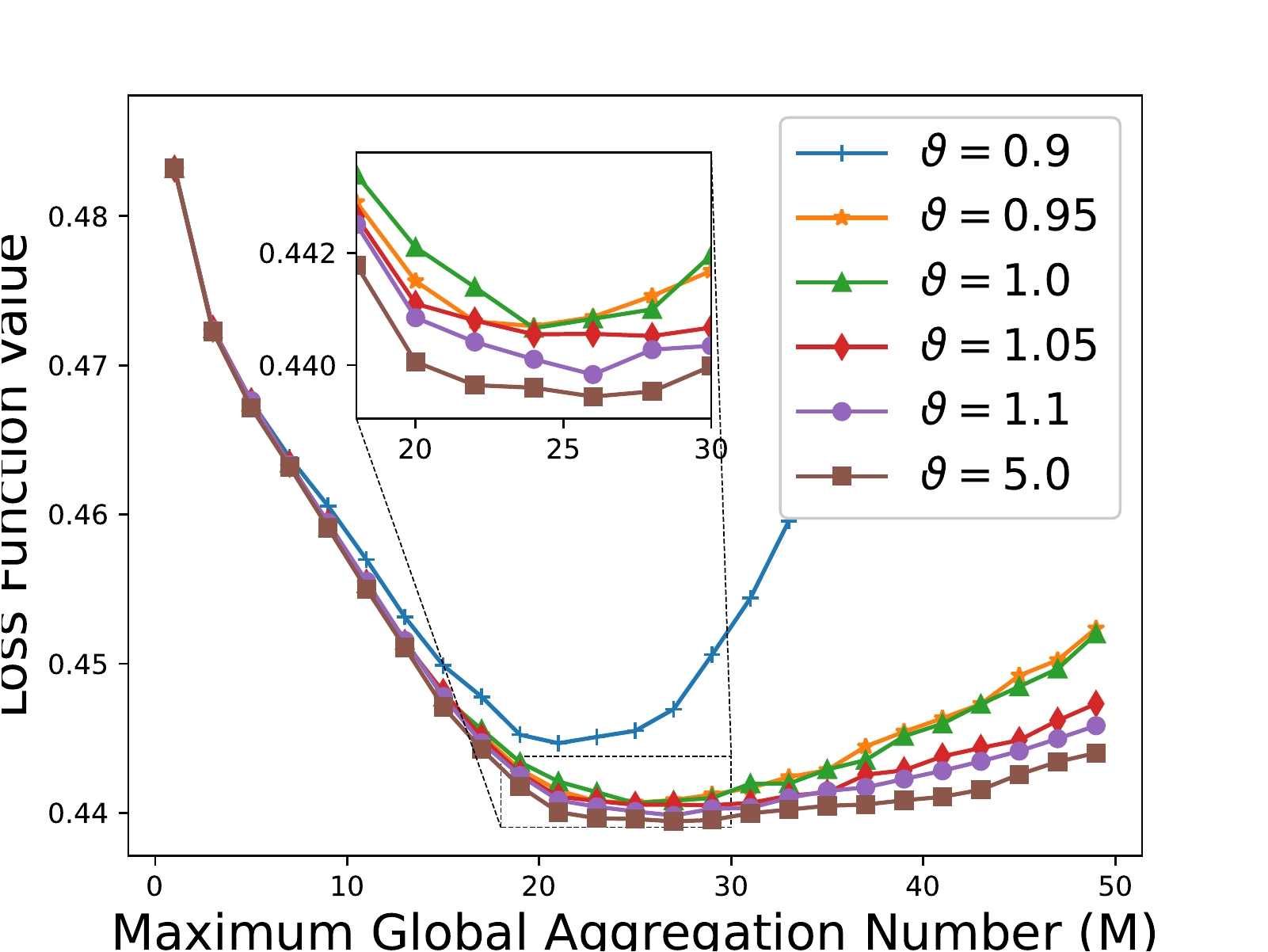}
			\caption{Loss function value vs. $M$}
			\sublabel{fig-svm-loss-noise-ep10}
		\end{subfigure}
		\begin{subfigure}[b]{0.24\textwidth}
			\centering
			\includegraphics[width=\textwidth]{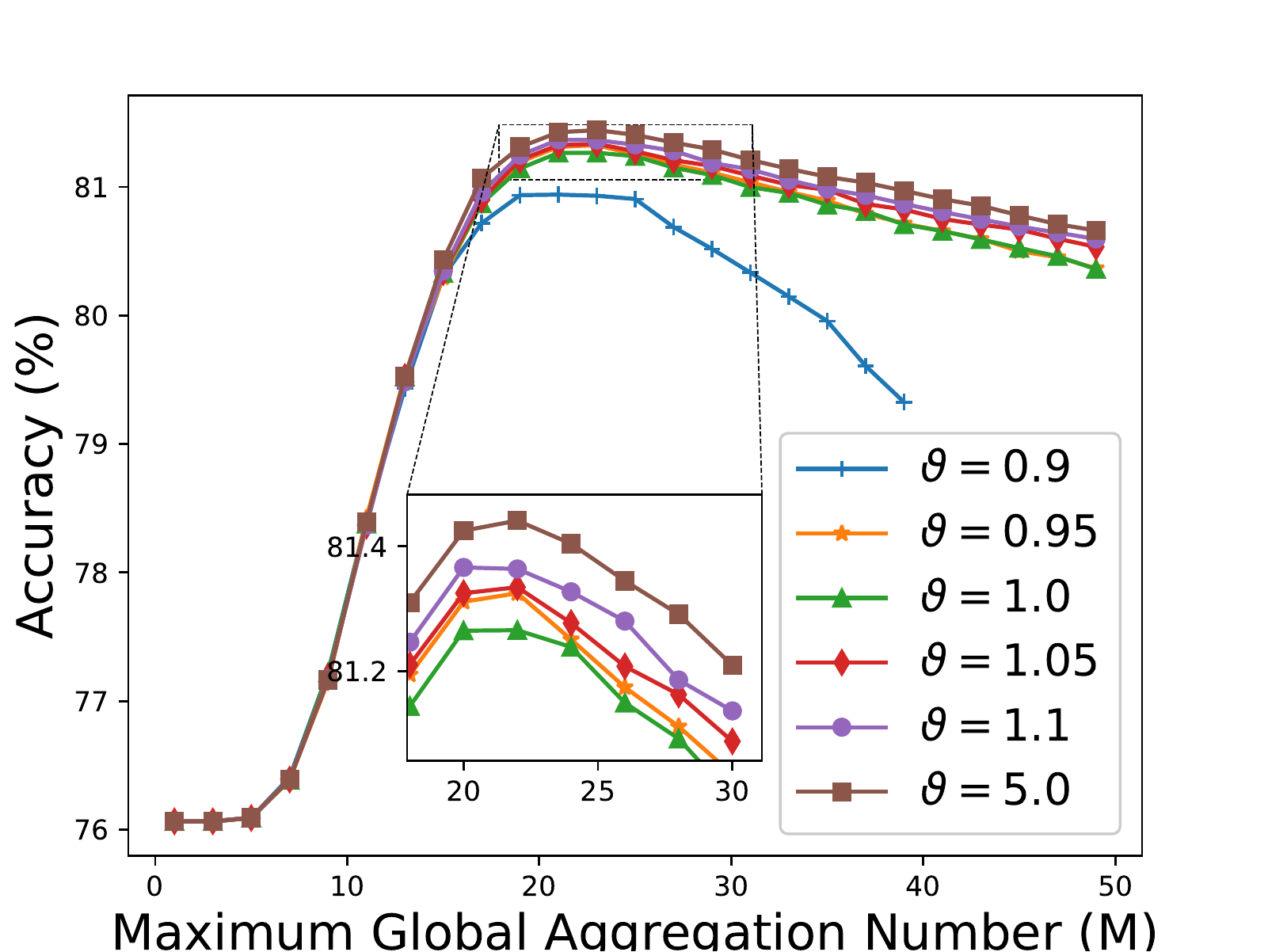}
			\caption{Accuracy vs. $M$}
			\sublabel{fig-svm-acc-noise-ep10}
		\end{subfigure}
		\caption{Loss function value and Accuracy of the SVM model on the ADULT dataset vs. the maximum number of global aggregations $M$ under different values of $\vartheta$, where $\epsilon = 10$.}
		\label{fig-noise-svm}
	\end{figure}
	\begin{figure}[!t]
		\centering
		\begin{subfigure}[b]{0.24\textwidth}
			\centering
			\includegraphics[width=\textwidth]{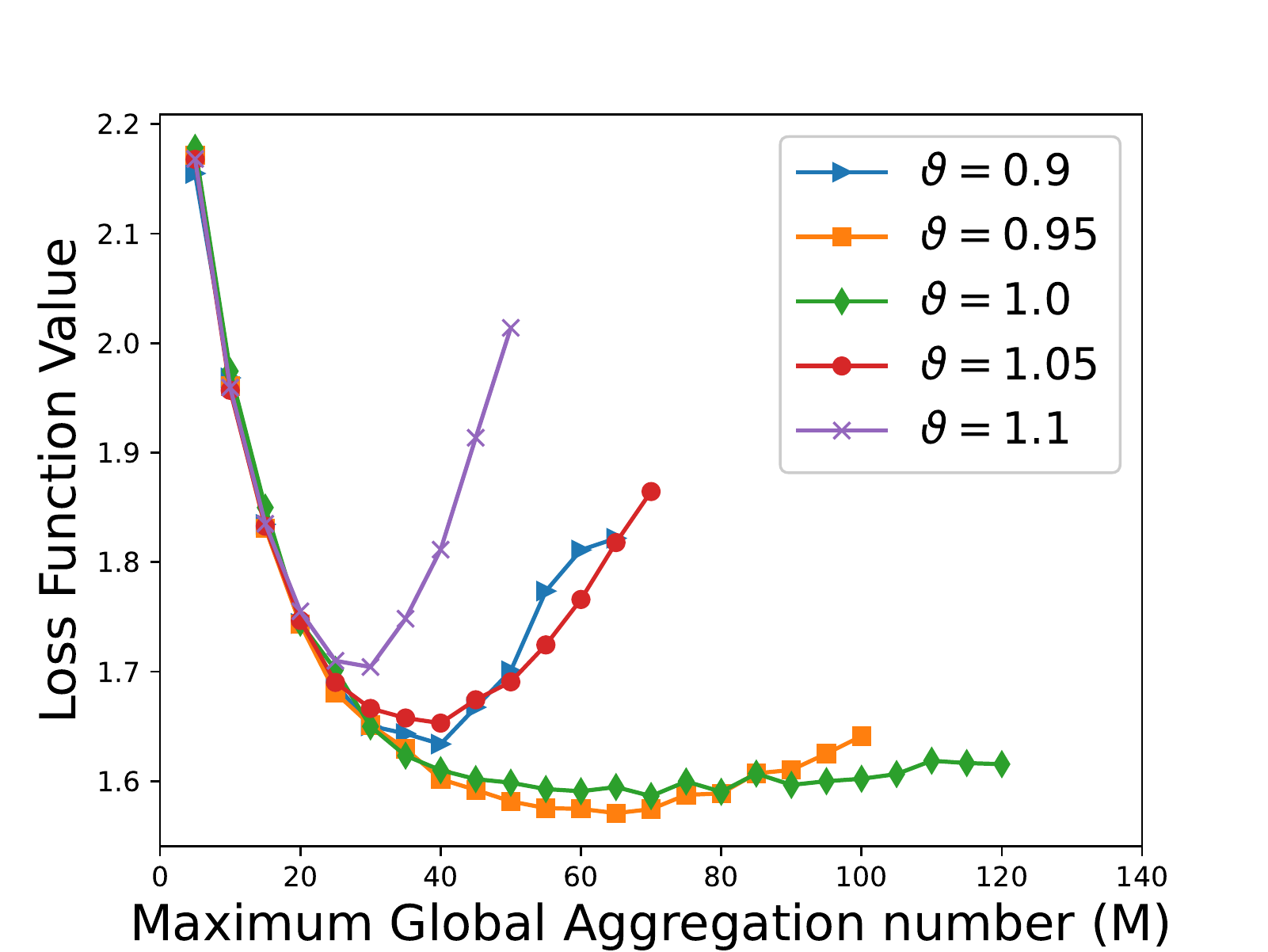}
			\caption{Loss function value vs. $M$ (CNN on CIFAR10)}
			\sublabel{fig-cifar-loss-privacy-ep10}
		\end{subfigure}
		\hfill
		\begin{subfigure}[b]{0.24\textwidth}
			\centering
			\includegraphics[width=\textwidth]{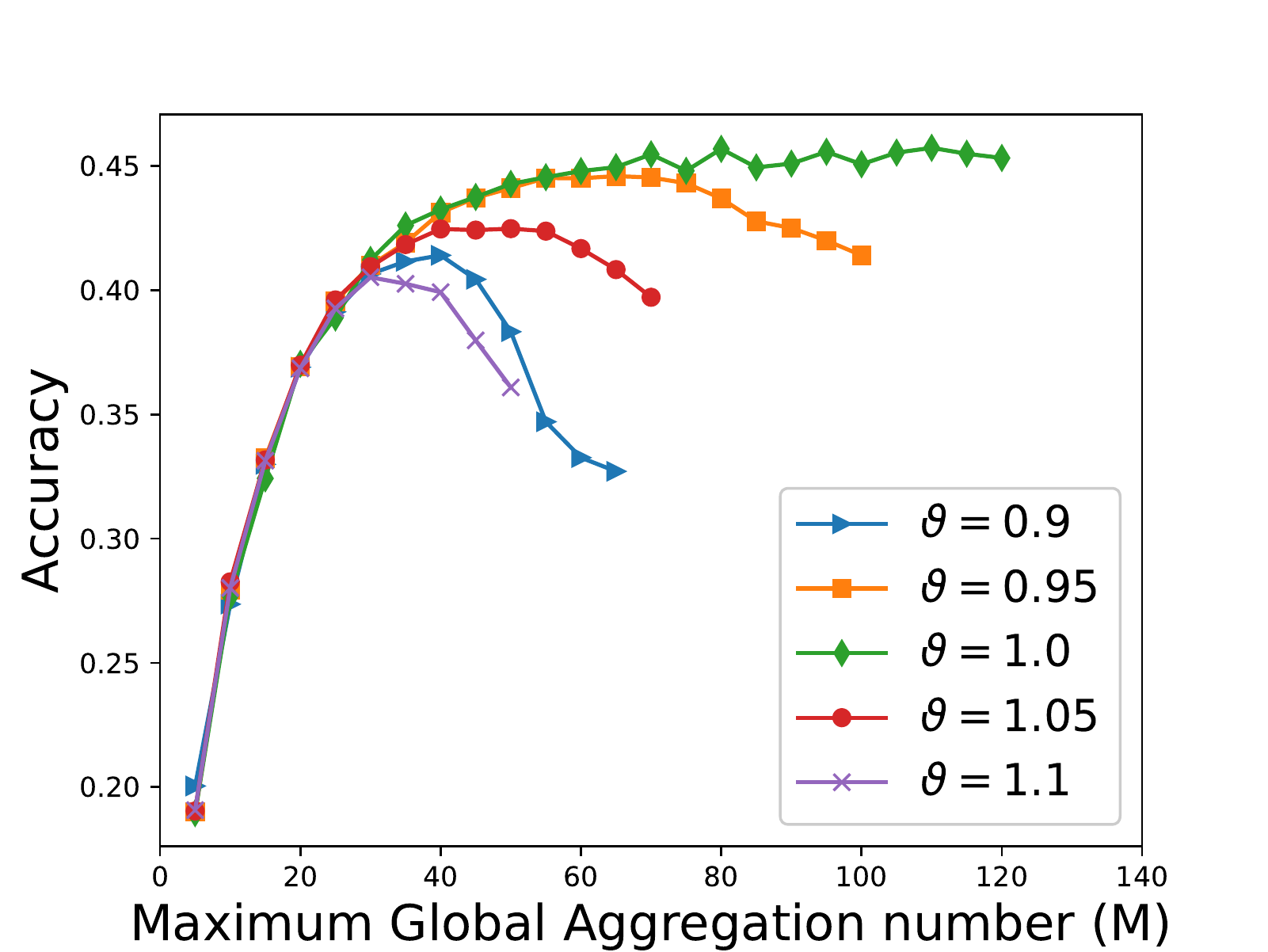}
			\caption{Accuracy vs. $M$ (CNN on CIFAR10)}
			\sublabel{fig-cifar-acc-privacy-ep10}
		\end{subfigure}
		\begin{subfigure}[b]{0.24\textwidth}
			\centering
			\includegraphics[width=\textwidth]{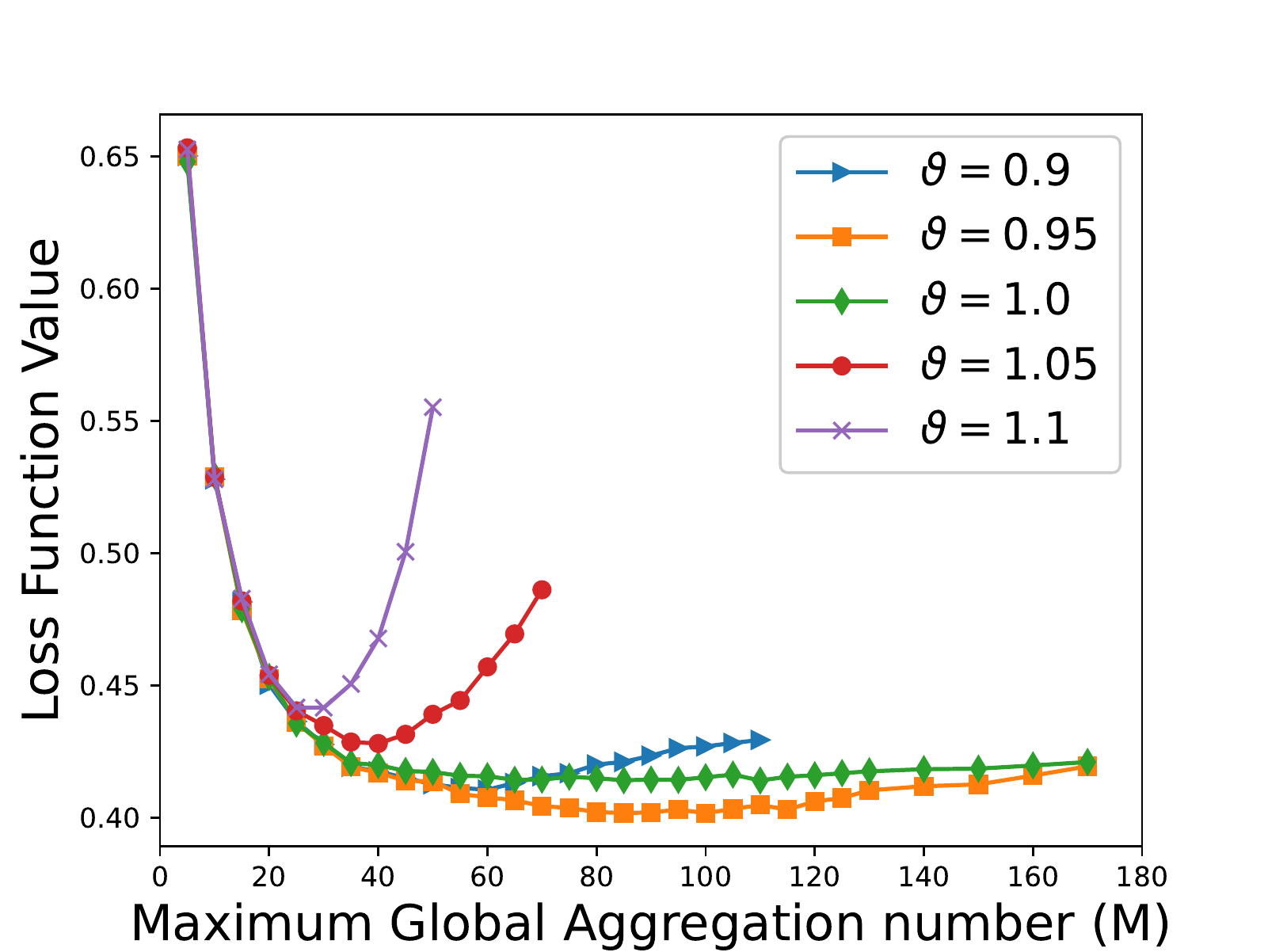}
			\caption{Loss function value vs. $M$ (CNN on FMNIST)}
			\sublabel{fig-fmnist-loss-privacy-ep10}
		\end{subfigure}
		\hfill
		\begin{subfigure}[b]{0.24\textwidth}
			\centering
			\includegraphics[width=\textwidth]{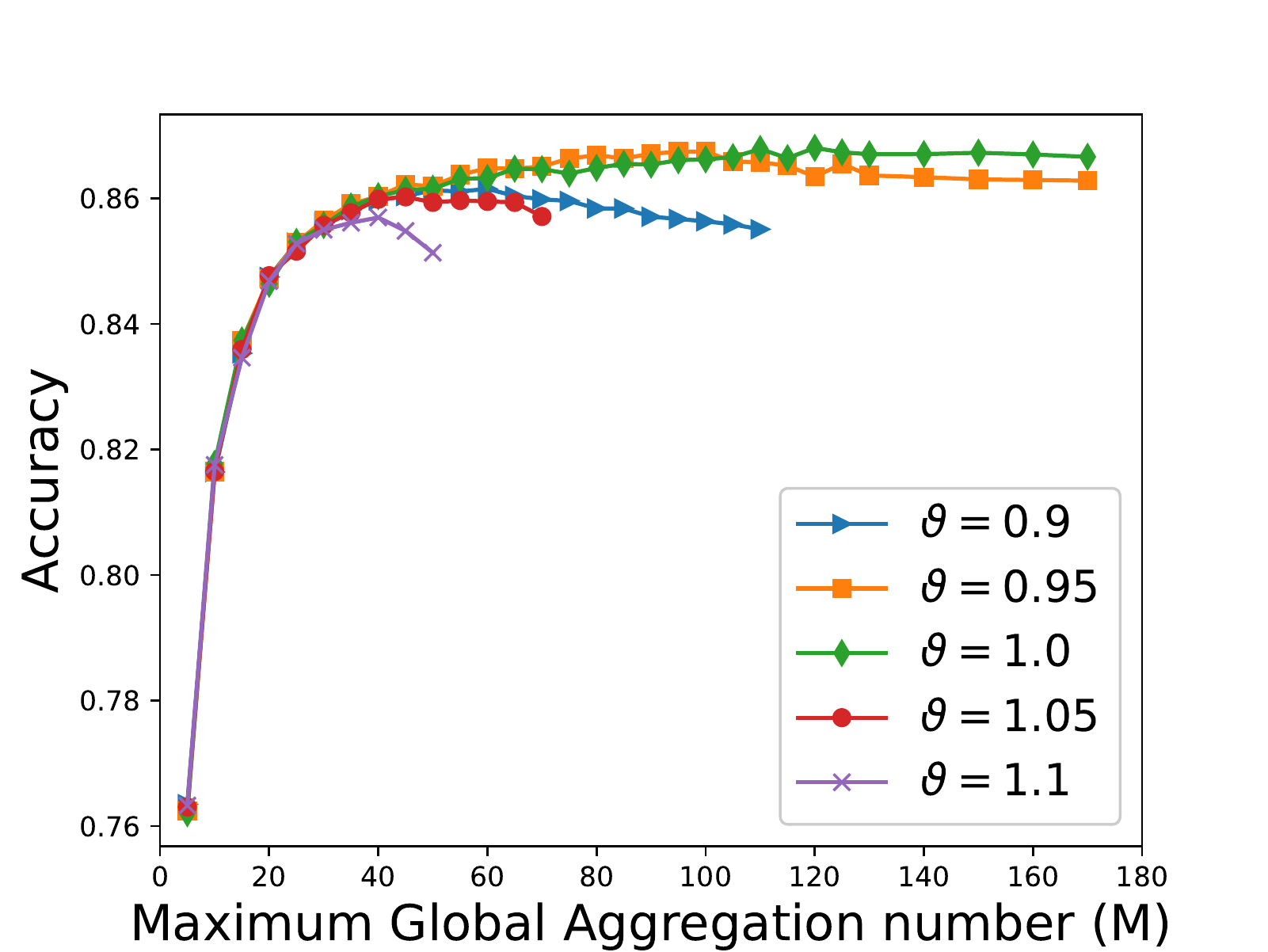}
			\caption{Accuracy vs. $M$ (CNN on FMNIST)}
			\sublabel{fig-fmnist-acc-privacy-ep10}
		\end{subfigure}
		\caption{Loss function value and Accuracy of the CNN model on the CIFAR10 and FMNIST datasets vs. $M$ under different values of $\vartheta $, where $\epsilon = 10$.}
		\label{fig-privacy-cifar-cnn}
	\end{figure}
	Figs.~\ref{fig-svm-loss-noise-ep10} and~\ref{fig-svm-acc-noise-ep10} plot the (testing) loss function and accuracy of the FL of the SVM model on the ADULT dataset under the proposed DP mechanism with time-varying perturbation noise variance. 
	Figs.~\ref{fig-cifar-loss-privacy-ep10}~and~\ref{fig-cifar-acc-privacy-ep10} plot those of the CNN model on the CIFAR10 dataset. Figs.~\ref{fig-fmnist-loss-privacy-ep10} and~\ref{fig-fmnist-acc-privacy-ep10} plot those of the CNN model on the FMNIST dataset. 
	
	Consistent with the observations made under the MLP model in Section~\ref{experiment: MLP}, there exists the optimal number of global aggregations, $M^*$, that minimizes the loss function of the SVM and CNN models (and maximizes their accuracy) while satisfying the $(\epsilon,\delta)$-DP privacy level. 
	The value of $\vartheta$ can also be configured to positively impact the utility of FL.
	
	On the other hand, Figs.~\ref{fig-svm-loss-noise-ep10} and~\ref{fig-svm-acc-noise-ep10} show that the time-increasing noise perturbation ($\vartheta>1$) achieves the best learning performance (i.e., achievable smallest loss and best accuracy), followed by the time-invariant noise perturbation ($\vartheta=1$) and then time-decreasing noise perturbation ($\vartheta<1$) on the SVM model. The time-decreasing noise perturbation ($\vartheta<1$) performs the best on the CNN models under both of the considered datasets.
	In contrast, $\vartheta>1$ is the best, followed by $\vartheta<1$, and $\vartheta=1$ is the worst on the MLP model; see Fig.~\ref{fig-loss-noise-M-lep5}.
	This is due to the distinct network architectures of the MLP, SVM, and CNN models.

	Fig.~\ref{fig-privacy-svm} illustrates the effect of increasing the maximum allowed number of global aggregations, $M$, on the loss and accuracy of the learning process for different values of the privacy parameter, $\epsilon$. The figure is based on the SVM model and the ADULT dataset, where $\vartheta = 1.05$ and $M = 50$. Each curve in the figure represents the results of a standalone training process with a given $\vartheta$ and $\epsilon$.
	It is observed in Fig.~\ref{fig-privacy-svm} that as $\epsilon$ increases, the loss function values decrease and approach the case with no DP noise perturbation (i.e., $\epsilon \to \infty$). Increasing $\epsilon$ can also lead to improved accuracy.
	It is also observed that the FL does not diverge under $\epsilon\rightarrow \infty$, as opposed to the rest of the $\epsilon$ values. 
	This is because when $\epsilon$ goes to infinity, no privacy is required and the FL considered is expected to behave like regular FedAvg.
	
	Fig.~\ref{fig-privacy-cnn-cifar} plots the (testing) loss and accuracy of the CNN models on the CIFAR10 and FMNIST datasets under different settings of the privacy level $\epsilon$.
	We see that the loss functions of the CNN models are also convex with respect to $M$, which is in line with Corollary~\ref{rema} and the observations made on the MLP model in Section~\ref{experiment: MLP}. 
	It is also seen that, as $\epsilon$ increases, the loss function values decrease.
	The optimal number of global aggregations $M^*$ also increases.
	This is also consistent with the observations made on the MLP model in Section~\ref{experiment: MLP}.

	\begin{figure}[!t]
		\centering
		\begin{subfigure}[b]{0.24\textwidth}
			\centering
			\includegraphics[width=\textwidth]{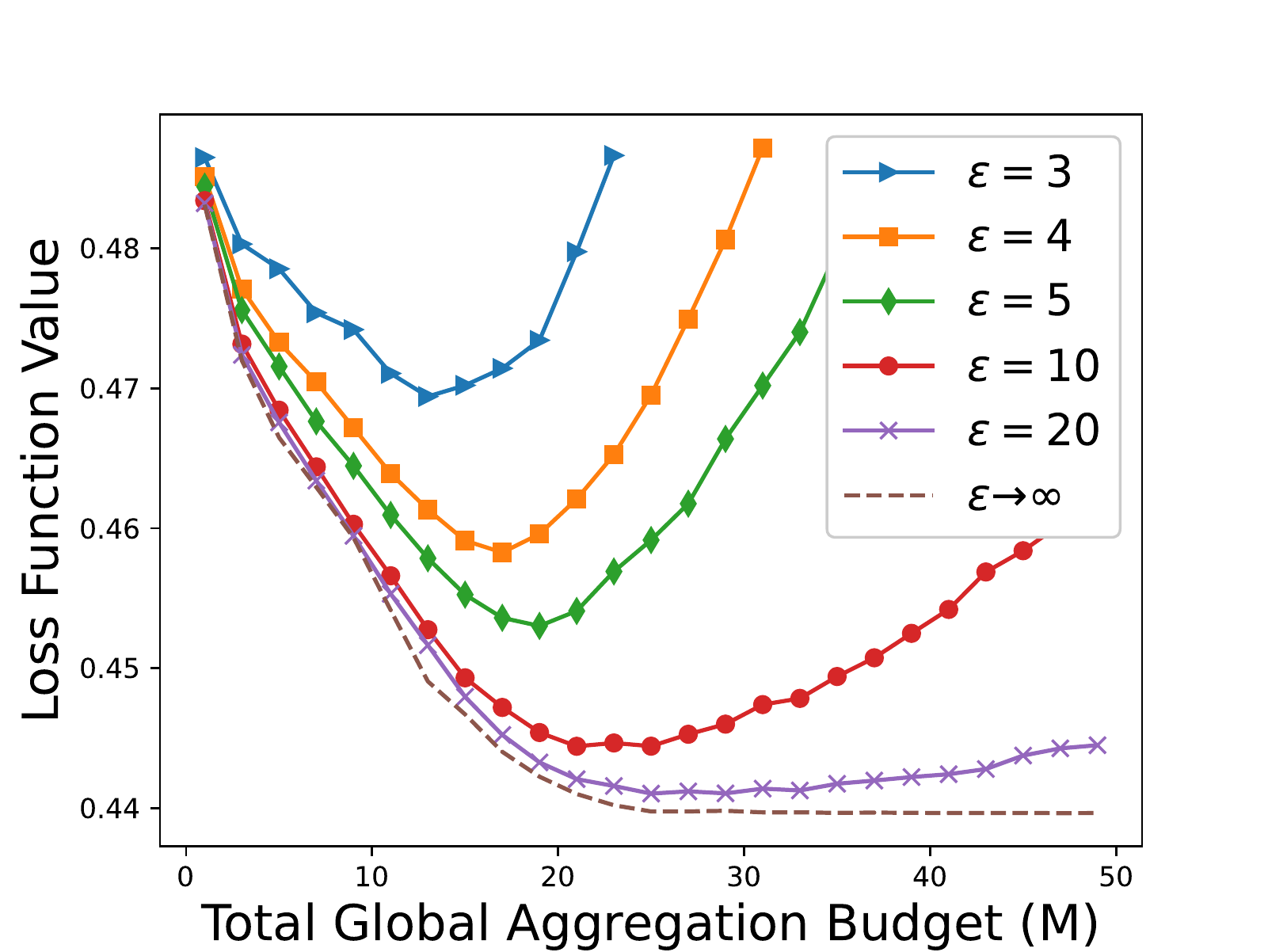}
			\caption{Loss function value vs. $M$}
			\sublabel{fig-svm-loss-privacy-ep10}
		\end{subfigure}
		\hfill
		\begin{subfigure}[b]{0.24\textwidth}
			\centering
			\includegraphics[width=\textwidth]{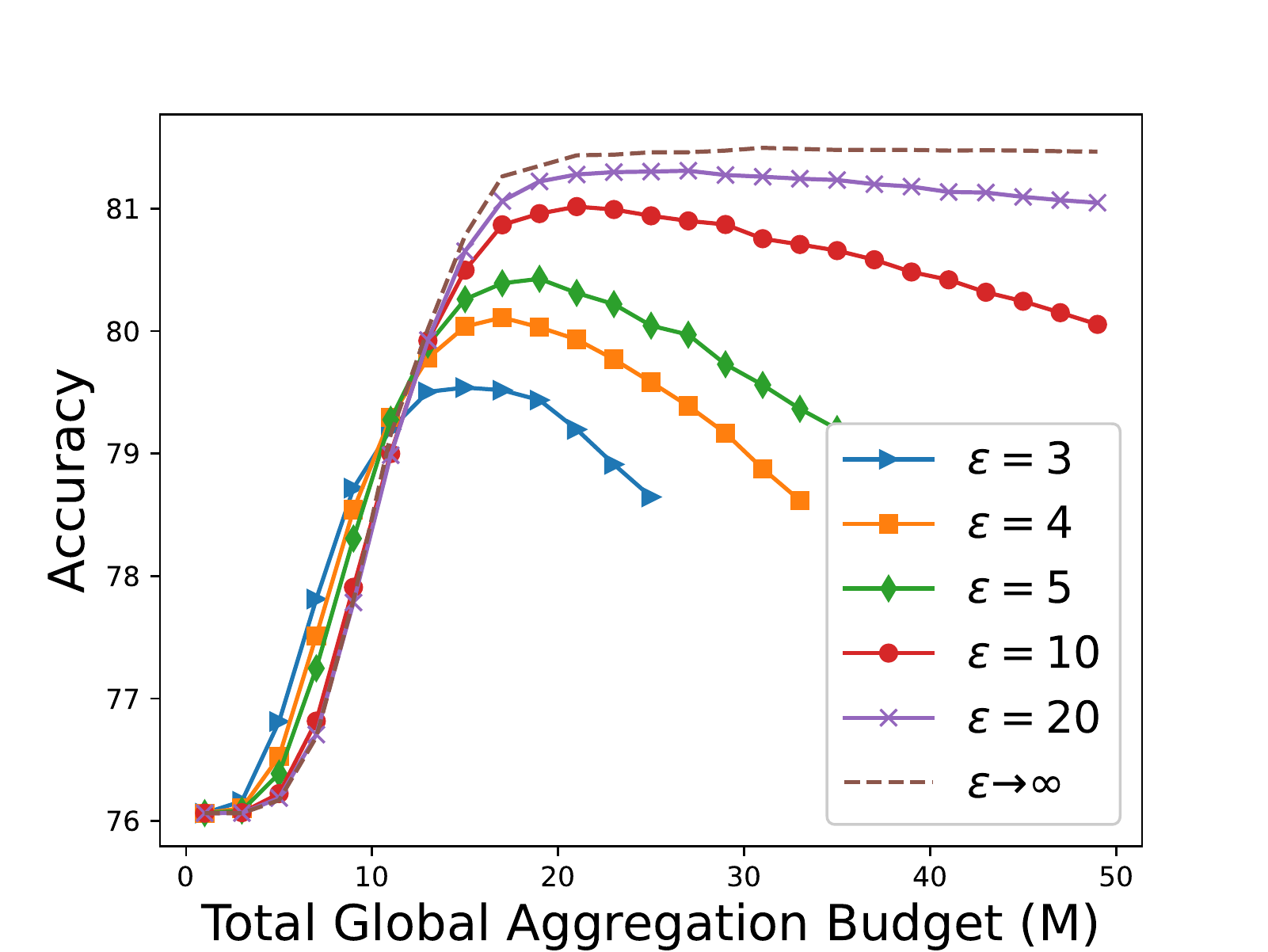}
			\caption{Accuracy vs. $M$}
			\sublabel{fig-svm-acc-privacy-ep10}
		\end{subfigure}
		\caption{Loss function value and Accuracy of the SVM model on the ADULT dataset vs. $M$ under different values of $\epsilon$, where $\vartheta = 1.05$.}
		\label{fig-privacy-svm}
	\end{figure}
	\begin{figure}[!t]
		\centering
		\begin{subfigure}[b]{0.24\textwidth}
			\centering
			\includegraphics[width=\textwidth]{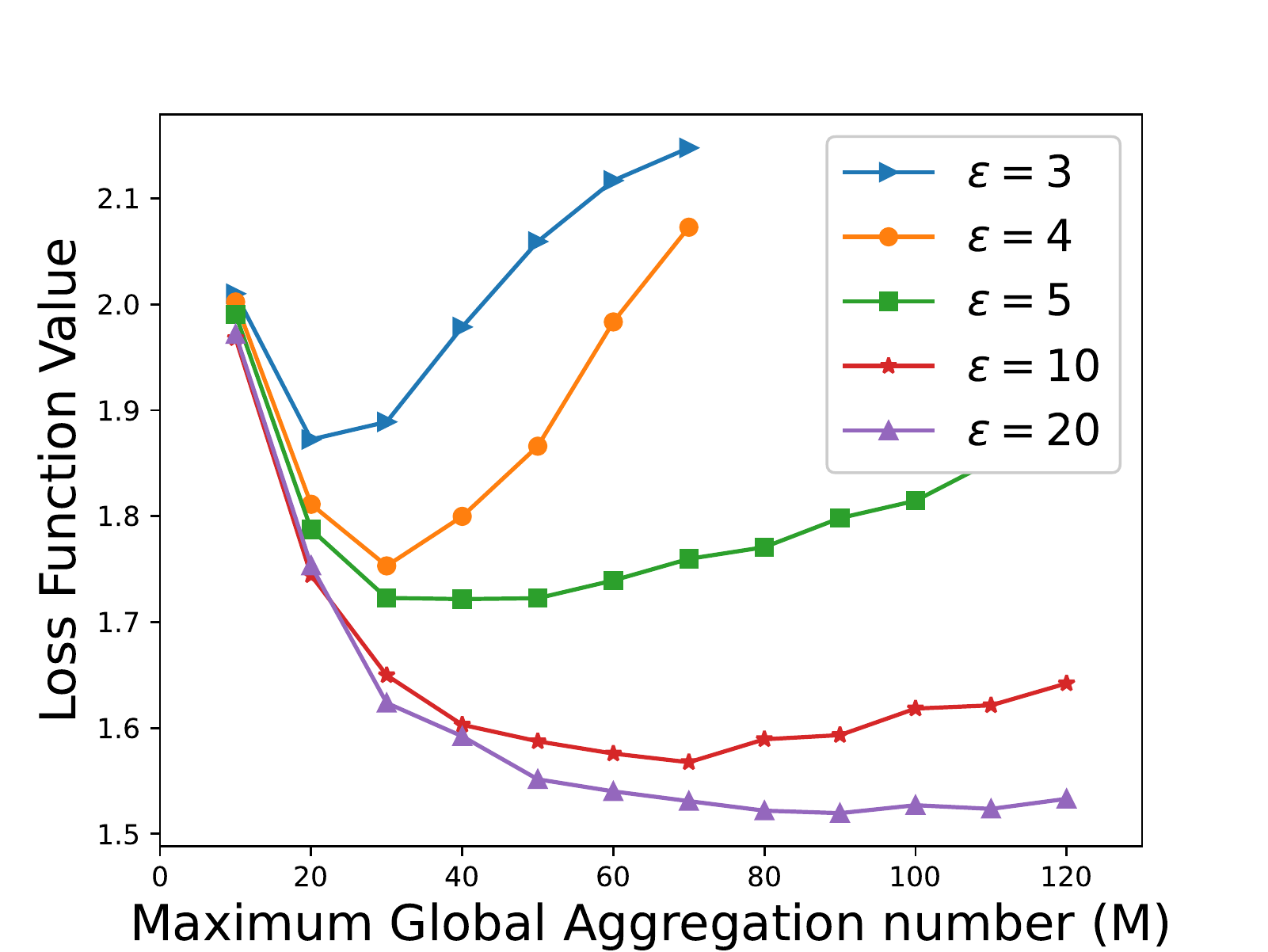}
			\caption{Loss function value vs. $M$ (CNN on CIFAR10).}
			\sublabel{fig-privacy-loss-cnn-cifar}
		\end{subfigure}
		\begin{subfigure}[b]{0.24\textwidth}
			\centering
			\includegraphics[width=\textwidth]{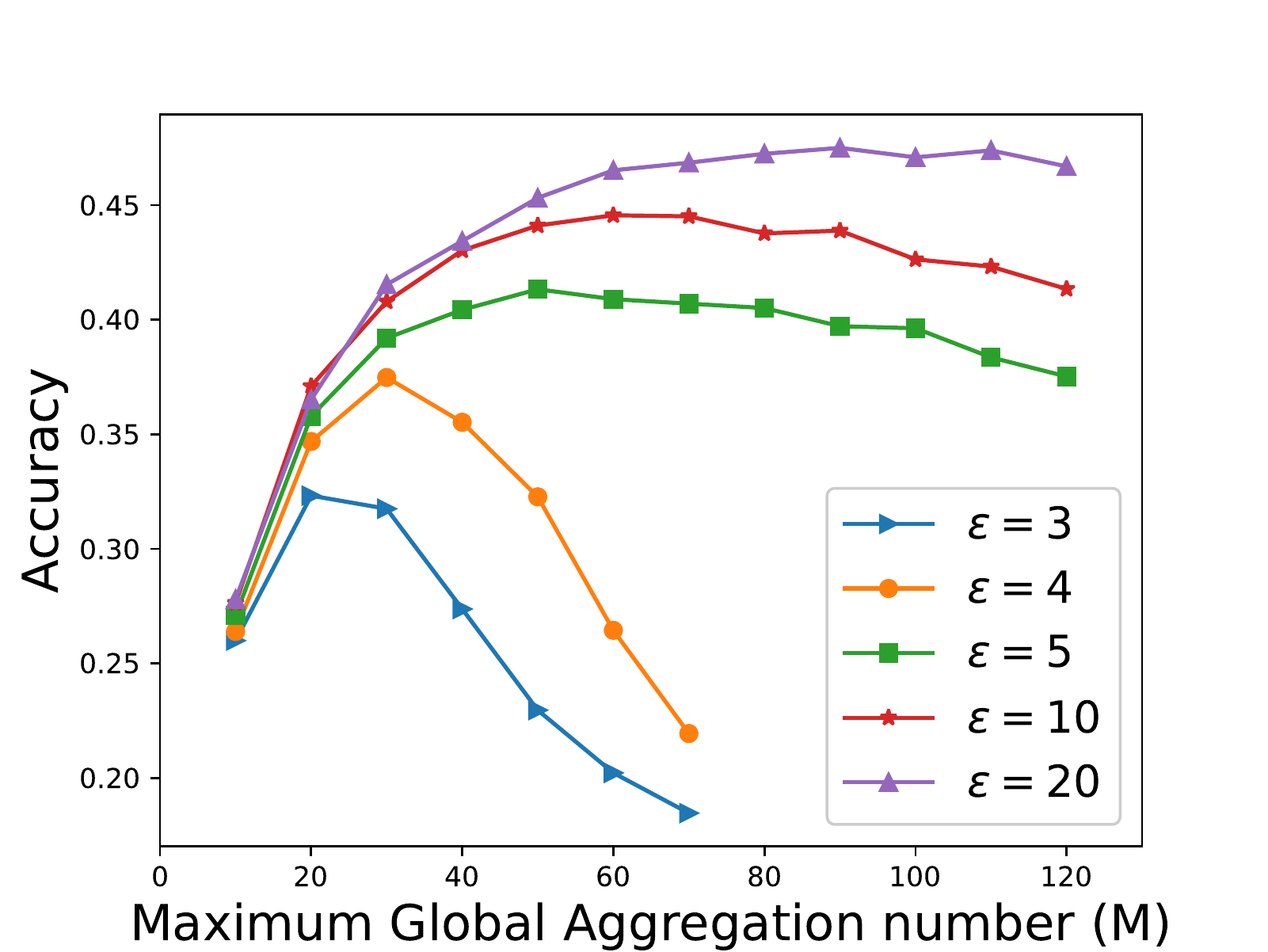}
			\caption{Accuracy vs. $M$ (CNN on CIFAR10).}
			\sublabel{fig-privacy-acc-cnn-cifar}
		\end{subfigure}
		\begin{subfigure}[b]{0.24\textwidth}
			\centering
			\includegraphics[width=\textwidth]{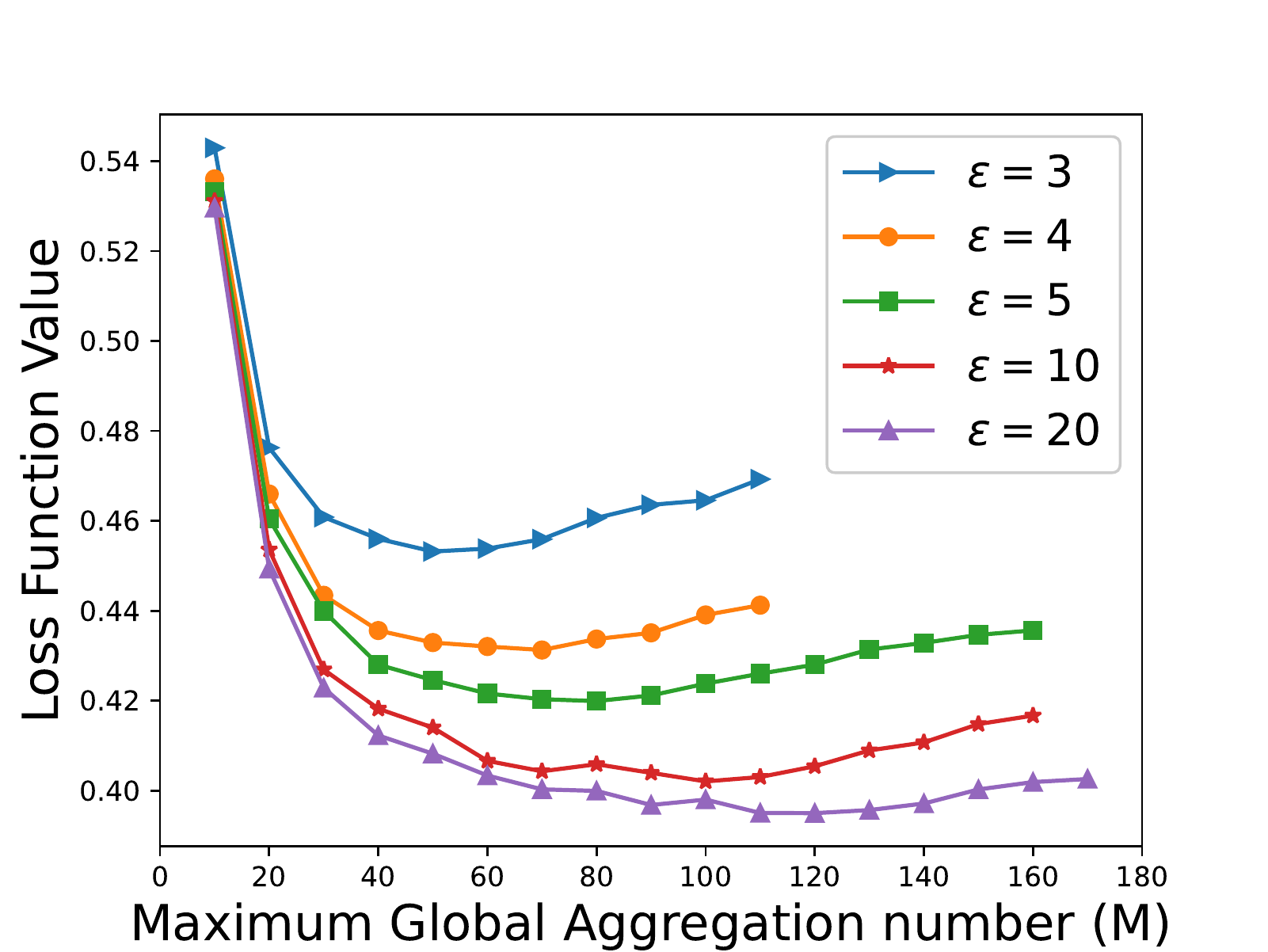}
			\caption{Loss function value vs. $M$ (CNN on FMNIST).}
			\sublabel{fig-privacy-loss-cnn-fmnist}
		\end{subfigure}
		\begin{subfigure}[b]{0.24\textwidth}
			\centering
			\includegraphics[width=\textwidth]{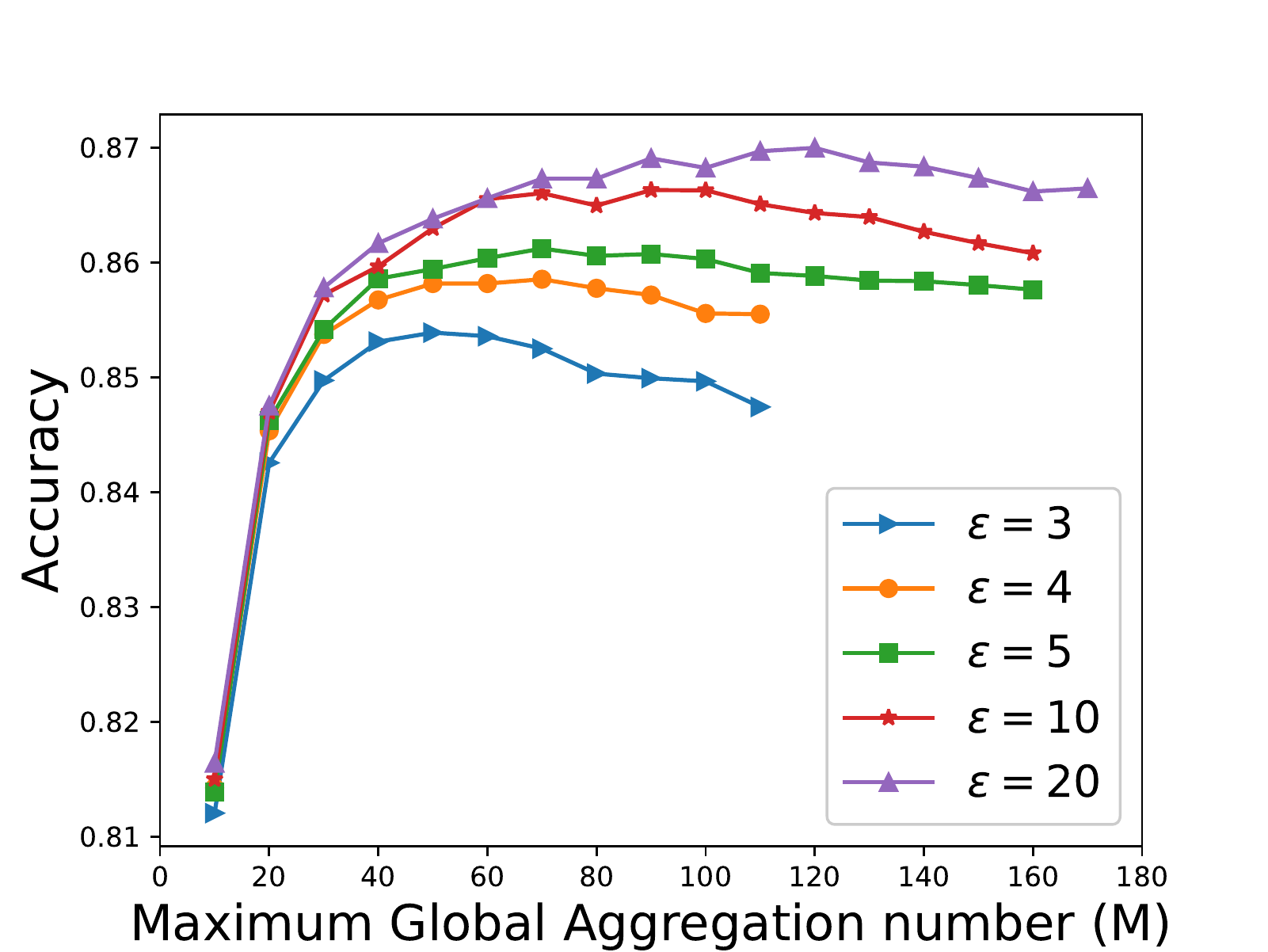}
			\caption{Accuracy vs. $M$ (CNN on FMNIST).}
			\sublabel{fig-privacy-acc-cnn-fmnist}
		\end{subfigure}
		\caption{Loss function value and Accuracy of the CNN model on the CIFAR10 and FMNIST datasets vs. the maximum number of global aggregations $M$ under different values of $\epsilon$, where 
			$\vartheta = 1.05$.}
		\label{fig-privacy-cnn-cifar}
	\end{figure}

	\section{Conclusion}\label{sec-conclusion}
	This paper has proposed and analyzed a new DP mechanism with a time-varying noise amplitude to balance the privacy and utility of FL.  
	We have established the varying amplitude as a function of the maximum number of global aggregations and the privacy protection levels. 
	We have also derived a convergence upper bound for the loss function of MLP models protected by the new mechanism, revealing a trade-off between the loss and privacy. 
	The number of global aggregations has been optimized based on the upper bound.
	Extensive experiments have assessed the convergence and utility of three different ML models trained using FL and protected by the new DP mechanism.
	The new DP mechanism with time-varying noise amplitudes has been seen to exhibit faster convergence and better accuracy under given privacy protection levels, compared to existing solutions.
	
	\appendices
	\section{Proof of Theorem~\ref{theo_DP noise}}\label{append-A}
	
	The privacy loss of a mechanism ${\cal M}$ is defined as~\cite{Abadi2016Deep}
	\begin{equation}
	loss_p = \exp(\alpha_{\cal{M}}(\lambda)),
	\end{equation}
	where $\alpha_{\cal{M}}(\lambda)$ is the $\lambda$-th moment, which is the logarithm of the moment generating function assessed at value $\lambda$.
	Based on the composability of the moment~\cite[Theorem~2]{Abadi2016Deep}, we have the $\lambda$-th moment for the time-varying Gaussian mechanism:
	\begin{equation}\label{eq-bound}
	\begin{aligned}
	\alpha_{\cal M}(\lambda) & \leq \sum_{m=1}^{M} \alpha_{{\cal M}_m}(\lambda,\sigma_m) = \sum_{m=1}^{M} \frac{q \lambda (\lambda +1)\Delta s^2}{2\sigma_m^2}\\
	& = \frac{q \lambda (\lambda +1)(\vartheta-\vartheta^{1-M})\Delta s^2}{2(\vartheta-1)\sigma^2}.
	\end{aligned}
	\end{equation}
	By exploiting the tail bound of the moment~\cite{Abadi2016Deep}, we can obtain the value of $\delta$ that satisfies the mechanism $\cal M$, i.e.,
	\begin{equation}\label{eq_tail_bound}
	\begin{aligned}
	\delta & = \underset{\lambda}{\min} \exp\left(\alpha_{\cal M}(\lambda) - \lambda \epsilon\right)\\
	& = \underset{\lambda}{\min} \exp\left(\frac{q \lambda (\lambda +1)(\vartheta-\vartheta^{1-M})(\Delta s)^2}{2(\vartheta-1)\sigma^2}- \lambda \epsilon\right).
	\end{aligned} 
	\end{equation}
	Let $g(\lambda) = \frac{q \lambda (\lambda +1)(\vartheta-\vartheta^{1-M})(\Delta s)^2}{2(\vartheta-1)\sigma^2}- \lambda \epsilon$. Since $\exp(\cdot)$ is a monotonically increasing function, the optimization problem in~\eqref{eq_tail_bound} is equivalent to finding the optimum of $\lambda$, denoted by $\lambda^*$, to minimize $g(\lambda)$. We derive the first-order derivative of $g(\lambda)$, as given by
	\begin{equation}
	g'(\lambda) = \frac{q (2\lambda +1)(\vartheta-\vartheta^{1-M})(\Delta s)^2}{2(\vartheta-1)\sigma^2}- \epsilon.
	\end{equation}
	By setting $g'(\lambda) = 0$, we obtain $\lambda^*$ as
	\begin{equation}\label{eq: optimal lambda}
	\lambda^* = \frac{\epsilon \sigma^2\left( \vartheta-1\right) }{q (\Delta s)^2\left(\vartheta - \vartheta^{1-M} \right)} - \frac{1}{2}.
	\end{equation}
	As a result, 
	\begin{equation}
	g(\lambda^*) \!= -\!\frac{q (\Delta s)^2(\vartheta\! -\! \vartheta^{1-M})}{2 \sigma^2(\vartheta-1)}\!\left(\!\frac{1}{2}-\!\frac{\epsilon \sigma^2 (\vartheta - 1)}{q(\Delta s)^2(\vartheta \!- \!\vartheta^{1 - M})} \!\right)^2. 
	\end{equation}
	The lower bound of $\delta$ can be given by
	\begin{equation}\label{eq_delta_LB}
	\delta \!\geq\! \exp\!\left[\!-\frac{q (\Delta s)^2(\vartheta\! -\! \vartheta^{1-M})}{2 \sigma^2(\vartheta-1)}\!\left(\!\frac{1}{2}\!-\!\frac{\epsilon \sigma^2 (\vartheta - 1)}{q(\Delta s)^2(\vartheta \!- \!\vartheta^{1-M})} \!\right)^2\!\right]\!.
	\end{equation}
	The RHS of \eqref{eq_delta_LB} provides the optimal value of $\delta$, denoted by~$\delta^*$.
	By taking the logarithm on both sides of~\eqref{eq_delta_LB}, we have
	\begin{equation}
	\begin{aligned}
	\ln (\delta) \!\geq\! -\frac{q (\Delta s)^2(\vartheta -\! \vartheta^{1-M})}{2 \sigma^2(\vartheta-1)}\!\left(\!\frac{1}{2}-\!\frac{\epsilon \sigma^2 (\vartheta - 1)}{q(\Delta s)^2(\vartheta -\! \vartheta^{1-M})} \!\right)^2.
	\end{aligned}
	\end{equation}
	Then, we have
	\begin{equation}\label{eq_delta}
	\begin{aligned}
	\ln \left(\!\frac{1}{\delta} \!\right) & \!\leq\! \frac{q (\Delta s)^2(\vartheta - \vartheta^{1-M})}{2 \sigma^2(\vartheta-1)}\left(\!\frac{1}{2}-\frac{\epsilon \sigma^2 (\vartheta - 1)}{\Delta s^2(\vartheta - \vartheta^{1-M})} \!\right)^2\\
	= & \frac{q (\Delta s)^2(\vartheta - \vartheta^{1-M})}{8 \sigma^2(\vartheta-1)}+\frac{\epsilon^2 \sigma^2 (\vartheta - 1)}{2q\Delta s^2(\vartheta - \vartheta^{1-M})} - \frac{\epsilon}{2}.
	\end{aligned}
	\end{equation}
	Since $\delta \in (0,1]$ and $\ln (\delta) \leq 0$, from~\eqref{eq_tail_bound} we have 
	\begin{equation}\label{eq_epsilon1}
	\frac{q \lambda (\lambda +1)(\vartheta-\vartheta^{1-M})(\Delta s)^2}{2(\vartheta-1)\sigma^2}- \lambda \epsilon \leq 0.	
	\end{equation}
	
	By substituting $\lambda^* $ in \eqref{eq: optimal lambda} into~\eqref{eq_epsilon1}, we have
	\begin{equation}\label{eq_epsilon2}
	\frac{q (\Delta s)^2(\vartheta - \vartheta^{1-M})}{8 \sigma^2(\vartheta-1)} \leq  {\frac{\epsilon}{4}}.
	\end{equation}
	Combining~\eqref{eq_delta} and~\eqref{eq_epsilon2} leads to  
	\begin{equation}
	\ln \left( \frac{1}{\delta}\right) \leq \frac{\epsilon^2 \sigma^2 (\vartheta - 1)}{2q(\Delta s)^2(\vartheta - \vartheta^{1-M})} - \frac{\epsilon}{4} \leq \frac{\epsilon^2 \sigma^2 (\vartheta - 1)}{2q\Delta s^2(\vartheta - \vartheta^{1-M})}.
	\end{equation}
	To ensure the $(\epsilon,\delta)$-DP of the considered FL system, we choose $\sigma$ that satisfies
	\begin{equation}
	\sigma^2 \leq \frac{2q(\Delta s)^2(\vartheta - \vartheta^{1-M})}{\epsilon^2 (\vartheta - 1)} \ln \left( \frac{1}{\delta}\right).
	\end{equation}
	This concludes the proof of the theorem.
	
	\section{Proof of Theorem~\ref{theorem-noise update}}\label{appendix-theo-5}
	Based on the definition and composability of the $\lambda$-th moment, we have the $\lambda$-th moment for the time-varying Gaussian mechanism, as given by
	\begin{subequations}\label{eq_alpha}
		\begin{align}
		&\alpha_{\cal M}(\lambda)  = \sum_{n=1}^{M'} \alpha_{{\cal{M}}_n}(\lambda,\sigma_n)\label{eq_alpha a} \\
		& = \sum_{n=1}^{m} \alpha(\lambda,\sigma_n) + \sum_{n=m+1}^{M'} \alpha(\lambda,\sigma_n) \label{eq_alpha b} \\
		& = \sum_{n=1}^{m} \frac{q \lambda (\lambda +1)(\Delta s)^2}{2\sigma_n^2} + \sum_{n=m+1}^{M'} \frac{q \lambda (\lambda +1)(\Delta s)^2}{2\sigma_n^2} \label{eq_alpha c}\\
		& \leq \left\{ \begin{aligned}
		& \frac{q \lambda (\lambda +1)(\Delta s)^2}{2} \left(\sum_{n=1}^{m} \frac{1}{\sigma_n^2} + \sum_{n=1}^{M'-m} \frac{1}{\sigma^2}\right)  ,\,{\text{if}~\vartheta \geq 1;}\\ 
		& \frac{q \lambda (\lambda +1)(\Delta s)^2}{2} \left(\sum_{n=1}^{m} \frac{1}{\sigma_n^2} + \frac{\vartheta^{m-M'}}{(1-\vartheta)\sigma^2}\right),\,{\text{if}~\vartheta < 1;} 
		\end{aligned}\right.\label{eq_alpha d}\\
		& = \left\{ \begin{aligned} 
		&\frac{q \lambda (\lambda \!+\!1)(\frac{\vartheta\!-\!\vartheta^{1-m}}{ \vartheta-1} + M' -m)(\Delta s)^2}{2\sigma^2},\,{\text{if}~\vartheta > 1;}\\
		&\frac{q \lambda (\lambda \!+\!1)M'(\Delta s)^2}{2\sigma^2},\,{\text{if}~\vartheta = 1;}\\
		&\frac{q \lambda (\lambda \!+\!1)({\vartheta^{1-m} - \vartheta} + \vartheta^{m-M'})(\Delta s)^2}{2(1- \vartheta)\sigma^2},\,{\text{if}~\vartheta < 1,}
		\end{aligned}\right.
		\end{align}
	\end{subequations}
	where~\eqref{eq_alpha d} is because $\sigma_n^2 = \vartheta^{n-1} \sigma^2 \leq \sigma^2$ if $\vartheta \geq 1$,
	and $\sum_{n=m+1}^{M'} \frac{1}{\sigma_n^2} \leq 
	\frac{1}{\vartheta}\frac{(\frac{1}{\vartheta})^{M'-m}-(\frac{1}{\vartheta})^{m+1}}{\frac{1}{\vartheta-1}} \leq \frac{\vartheta^{m-M'}}{1- \vartheta} $ if $\vartheta < 1$.
	
	By exploiting the tail bound of the moment~\cite{Abadi2016Deep}, we can obtain $\delta$ that satisfies the mechanism $\cal M$, as given by
	\begin{subequations}\label{eq_tail_bound 1}
		\begin{align}
		&\delta  = \underset{\lambda}{\min} \exp\left(\alpha_{\cal M}(\lambda) - \lambda \epsilon\right)\\
		& = \left\{ \begin{aligned}
		& \!\underset{\lambda}{\min} \exp\!\left[\!\frac{q \lambda (\lambda \!+\!1)(\Delta s)^2}{2\sigma^2}\!\left(\!\frac{\vartheta\!-\!\vartheta^{1- m}}{\vartheta-1} \!+\! M' \!-\! m \!\right) \!-\! \lambda \epsilon \right]\!,\\
		&\qquad\qquad\qquad\qquad\qquad\qquad\qquad\qquad\qquad\;\;{\text{if}~\vartheta > 1;}\\
		& \!\underset{\lambda}{\min} \exp\!\left[\!\frac{q \lambda (\lambda \!+\!1)M'(\Delta s)^2}{2\sigma^2}  - \lambda \epsilon \right]\!,\,{\text{if}~\vartheta = 1;}\\
		& \!\underset{\lambda}{\min} \exp\!\left[\!\frac{q \lambda (\lambda \!+\!1)(\Delta s)^2}{2\sigma^2}\!\left(\!\frac{\vartheta^{1- m} \!-\!\vartheta + \vartheta^{m-M'}}{\vartheta-1} \!\right) \!-\! \lambda \epsilon \right]\!,\\
		&\qquad\qquad\qquad\qquad\qquad\qquad\qquad\qquad\qquad\;\;{\text{if}~\vartheta < 1.}
		\end{aligned}\right.
		\end{align} 
	\end{subequations}
	
	Following the steps in the proof of Theorem~\ref{theo_DP noise}, we set $\lambda^* = \frac{\epsilon \sigma^2\left( \vartheta-1\right)}{q (\Delta s)^2\left[\vartheta - \vartheta^{1 - m}+ (M' - m)(\vartheta - 1) \right]} - \frac{1}{2}$ if $\vartheta \geq 1$ and $\lambda^* = \frac{\epsilon \sigma^2\left( \vartheta-1\right)}{q (\Delta s)^2\left[\vartheta - \vartheta^{1 - m}+ \vartheta^{m-M'} \right]} - \frac{1}{2}$ if $\vartheta < 1$. Then, the amplitude of the noise satisfying the $(\epsilon,\delta)$-DP of the training can be obtained in \eqref{eq_DP noise re}.
	
	\section{Proof of Theorem~\ref{theo_convergence bound}}\label{appendix_convergence bound}
	The proof starts by defining 
	\begin{equation}\label{eq_def_global_weight}
	{\bm \omega'}{(t)} \! \triangleq \! \sum_{k \in {\cal K}} p_k\left(\! {\bm \omega}_k(t)+{\bm{n}}_k(t)\!\right) \!=\! \sum_{k \in {\cal K}} p_k {\bm \omega}_k(t)+{\bm{n}}(t),
	\end{equation}
	where ${\bm n}(t) = \sum_{k \in {\cal K}} p_k {\bm{n}}_k(t)$.
	
	Based on the $L$-Lipschitz smoothness of the global loss function $F(\cdot)$ and the Taylor expansion, it follows that
	\begin{equation}\label{eq_taylor_expansion}
	\begin{aligned}
	F({\bm \omega'}(t+1)) -F({\bm \omega'}(t)) 
	& \leq \nabla F\left({\bm \omega'}(t) \right) \left({\bm \omega'}(t+1) - {\bm \omega'}(t)  \right) \\
	&\quad + \frac{L}{2} \left\| {\bm \omega'}(t+1) - {\bm \omega'}(t)  \right\|^2.
	\end{aligned}
	\end{equation} 
	According to the gradient descent, we have
	\begin{equation}\label{eq_gsd}
	{\bm \omega'}_k(t+1) = {\bm \omega'}_k(t) - \eta \nabla f_k({\bm \omega'}(t) ).
	\end{equation}
	By substituting~\eqref{eq_gsd} and~\eqref{eq_def_global_weight} into~\eqref{eq_taylor_expansion}, we obtain~\eqref{eq_F_w_t1}.
	
	\begin{figure*}[ht]
		\begin{subequations}\label{eq_F_w_t1}
			\begin{align}
			F \left( {\bm \omega'}(t\!  +\! 1)\!  \right) &- F\left(\! {\bm \omega'}(t)  \right) 
			\! \leq  \nabla F\left({\bm \omega'}(t)\right)\! \bigg(\! {\bm n}(t+1) - \eta\sum_{k \in {\cal K}^m} p_k \nabla f_k\left({\bm \omega'}(t)\right) \!\bigg) \! +\!  \frac{L}{2} \bigg\|\! \sum_{k \in {\cal K}^m} p_k \left( {\bm n}_k(t + 1)- \eta \nabla f_k\left({\bm \omega'}_k(t) \right) \right)\! \bigg\|^2\\
			&=  \frac{\eta^2L}{2} \bigg\|\sum_{k \in {\cal K}^m} p_k \nabla f_k\left({\bm \omega'}(t) \right) \bigg\|^2 - \eta \nabla F\left({\bm \omega'}(t) \right) \sum_{k \in {\cal K}^m} p_k \nabla f_k\left({\bm \omega'}(t) \right) + \frac{L}{2} \bigg\|\sum_{k \in {\cal K}^m} p_k  {\bm n}_k(t+1)\bigg\|^2 
			\end{align}
		\end{subequations}
		\hrulefill
	\end{figure*}
	
	Next, we take the expectation on both sides of \eqref{eq_F_w_t1} with respect to $K$ randomly chosen users out of the $U$ users at the $m$-th global aggregation, i.e., ${\cal K}^m$, and obtain
	\begin{subequations}\label{eq_F_w_t1_avg}
		\begin{align}
		&\mathbb{E}_{{\cal K}^m}\!\left\lbrace\! F\!\left(\!{\bm \omega'}(t\!+\!1) \!\right)\! \right\rbrace \!
		\!\leq \! F\left({\bm \omega'}(t) \right) - \eta \left\|\nabla F\!\left({\bm \omega'}(t)\right)\! \right\|^2  \\
		& \!+\! \frac{\eta^2L}{2} \mathbb{E}_{{\cal K}^m}\!\bigg\lbrace\!\Big\|\! \sum_{k \in {\cal K}^m}\! p_k \nabla f_k\!\left(\!{\bm \omega'}(t)\!\right)\! \Big\|^2 \!\bigg\rbrace \!+\! \frac{L}{2} \mathbb{E}_{{\cal K}^m}\!\left\lbrace\!\left\|{\bm n}(t\!+\!1)\right\|^2 \!\right\rbrace\!.
		\end{align}
	\end{subequations}
	Since $p_k = \frac{|{\cal D}_k|}{|{\cal D}|} = \frac{1}{K}$,  $\mathbb{E}_{{\cal K}^m}\{\| \sum_{k \in {\cal K}^m} p_k \nabla f_k({\bm \omega'}(t))\|^2 \}$ in \eqref{eq_F_w_t1_avg} is rewritten as \eqref{eq_delta_fk}.
	\begin{figure*}[ht]
		\begin{subequations}\label{eq_delta_fk}
			\begin{align}
			\mathbb{E}_{{\cal K}^m}\Bigg\lbrace\Big\| \sum_{k \in {\cal K}^m} p_k & \nabla f_k\left({\bm \omega'}(t) \right)\Big\|^2 \Bigg\rbrace  = \frac{1}{UK}\left\|\sum_{k \in {\cal U}} \nabla f_k\left({\bm \omega'}(t)\right)\right\|^2\\
			& = \frac{1}{UK} \sum_{k \in {\cal U}}\left\| \nabla f_k\left({\bm \omega'}(t) \right)\right\|^2  + \frac{1}{UK}\frac{K-1}{U-1}\sum_{k \in {\cal U}}\sum_{j \in {\cal U}/k}\left[\nabla f_k\left({\bm \omega'}(t) \right) \right]^T\nabla f_j\left({\bm \omega'}(t) \right)\\
			& = \frac{U-K}{UK(U-1)}\sum_{k \in {\cal U}}\left\| \nabla f_k\left({\bm \omega'}(t) \right)\right\|^2 + \frac{K-1}{UK(U-1)}\left\| \nabla F\left({\bm \omega'}(t) \right)\right\|^2.
			\end{align}
		\end{subequations}
		\hrulefill
	\end{figure*}
	By the definition of $\cal B$-local dissimilarity, $\mathbb{E}_{{\cal K}^m} [\left\| \nabla f_k({\bm \omega'})\right\|^2 ] \leq {\cal B}^2 \left\| \nabla F({\bm \omega'}) \right\|^2$, we obtain
	\begin{subequations}\label{eq_delta_F_w}
		\begin{align}
		&\mathbb{E}_{{\cal K}^m}\left\lbrace\left\| \sum_{k \in {\cal K}^m} p_k \nabla f_k\left({\bm \omega'}(t)\right) \right\|^2 \right\rbrace \notag\\
		&\; \leq \left[ \frac{(U-K){\cal B}^2}{K(U-1)} + \frac{K-1}{UK(U-1)}\right] \left\| \nabla F\left({\bm \omega'}(t)\right)\right\|^2.\tag{\ref{eq_delta_F_w}}
		\end{align}
	\end{subequations}
	
	Since $\mathbb{E}_{{\cal K}^m}\lbrace F({\bm \omega'}(t+1) ) \rbrace =F({\bm \omega'}(t+1)) $, substituting \eqref{eq_delta_F_w} into~\eqref{eq_F_w_t1_avg} yields 
	\begin{subequations}\label{eq_F_w1 2}
		\begin{align}
		F({\bm \omega'}(t+1)) & \leq F({\bm \omega'}(t)) \!+\! \phi \|\nabla F({\bm \omega'}(t)) \|^2 \notag\\
		& \quad  + \frac{L}{2} \mathbb{E}_{{\cal U}^t}\lbrace\left\|{\bm n}(t\!+\!1) \right\|^2 \rbrace, \tag{\ref{eq_F_w1 2}}
		\end{align}
	\end{subequations}
	where $\phi \triangleq \frac{\eta^2 L}{2}\left( \frac{(U-K){\cal B}^2}{K(U-1)} + \frac{K-1}{UK(U-1)} \right) - \eta$.
	
	Subtracting $F({\bm \omega}^*)$ from both sides of \eqref{eq_F_w1 2} gives
	\begin{equation}\label{eq_Fw1_Fw}
	\begin{aligned}
	&F({\bm \omega'}(t+1) ) - F({\bm \omega}^*)
	\leq  F({\bm \omega'}(t) ) - F({\bm \omega}^*)\\ 
	&\qquad + \phi \|\nabla F({\bm \omega'}(t)) \|^2 
	+ \frac{L}{2} \mathbb{E}_{{\cal K}^m}\lbrace\|{\bm n}(t+1) \| \rbrace.
	\end{aligned}
	\end{equation}
	
	Considering Polyak-Lojasiewicz condition
	and $\mathbb{E}_{{\cal K}^m}\lbrace\left\|{\bm n}(t+1) \right\|^2 \rbrace = {\vartheta}^{m} \mathbb{E}_{{\cal K}^m}\lbrace\left\|{\bm n}(1) \right\|^2\rbrace = {\vartheta}^{m} \sigma^2$, we obtain
	\begin{equation}\label{eq_delta_Fw}
	\|\nabla F({\bm \omega'}(t+1)) \|^2
	\leq \left( 1+ 2\rho \phi \right) \|\nabla F({\bm \omega'}(t) ) \|^2+ {\rho L}{\vartheta}^{m} \sigma^2.
	\end{equation}
	
	Based on the recurrence expression~\eqref{eq_delta_Fw}, we can obtain the upper bound of $\|\nabla F({\bm \omega'}(t)) \|^2$, as given by
	\begin{subequations}
		\begin{align}
		\left\|\nabla F\left({\bm \omega'}(t)\right) \right\|^2 & \leq \left(1+2\rho \phi \right)^m \left\|\nabla F\left({\bm \omega'}(0)\right) \right\|^2 \\
		& \quad + \frac{\rho L\left[\vartheta^m-\left(1+ 2\rho \phi \right)^m\right] }{\left( \vartheta - 1- 2\rho \phi\right)}\sigma^2.
		\end{align}
	\end{subequations}
	Using the Polyak-Lojasiewicz condition again gives
	\begin{subequations}
		\begin{align}
		F({\bm \omega'}(t) )\!- F({\bm \omega}^*) & \!\leq\! \left[\! F({\bm \omega'}(0)) - F({\bm \omega}^*)\!\right]\! \left(2\rho \phi +1 \right)^m  \\
		& \quad + \frac{L\left[\vartheta^m-\left(2\rho \phi +1 \right)^m\right] }{2\left( \vartheta - 2\rho \phi -1 \right)}\sigma^2,
		\end{align}
	\end{subequations}
	which concludes this proof. 
	
	\section{Proof of Corollary~\ref{rema}}\label{append_coro}
	Define the RHS of \eqref{eq_convergnece_bound} to be $g(T)$ for the brevity of notation. The second-order derivative of $g(T)$ regarding $T$ is 
	\begin{subequations}\label{eq_partial}
		\begin{align}
		& \frac{\partial^2 g(T)}{\partial T^2} = \frac{\Theta}{\tau} {\cal A}^{\frac{T}{\tau}} \ln^2\left({\cal A}\right) + 
		\frac{qL(\Delta s)^2\ln\left(\frac{1}{\delta}\right)\vartheta}{\tau^2\epsilon^2 \left(\vartheta-{\cal A} \right)\left(U-1\right)}\notag\\
		& \quad \times \left(\vartheta^{\frac{T}{\tau}} \ln^2\left(\vartheta\right) -{\cal A}^{\frac{T}{\tau}} \ln^2\left({\cal A}\right)+\left(\frac{{\cal A}}{\vartheta}\right)^{\frac{T}{\tau}}\ln^2\left(\frac{{\cal A}}{\vartheta}\right)  \right)\notag\\
		& \quad + L \frac{\partial^2 {\cal H}\left(\frac{T}{M}\right)}{\partial T^2},\tag{\ref{eq_partial}}
		\end{align}		 
	\end{subequations}
	where the third term on the RHS of \eqref{eq_partial} is the second-order derivative of ${\cal H}\left(\frac{T}{M}\right)$ with respect to $T$, i.e.,
	\begin{equation*}
	\begin{aligned}
	\frac{\partial^2 {\cal H}\left(\frac{T}{M}\right)}{\partial T^2} & = \frac{\gamma }{M^2} \left(\eta L + 1 \right)^{T/M}\ln^2\left(\eta L + 1 \right) \geq 0.
	\end{aligned}
	\end{equation*}
	We see that the first and the third terms on the RHS of \eqref{eq_partial} are positive, and the second term is positive when $\vartheta \geq {\cal A}$. As a result, $\frac{\partial^2 g(T)}{\partial T^2} > 0$ if $\vartheta \geq {\cal A}$, and the upper bound is a convex function of~$T$. 
	
	By substituting $M = \frac{T}{\tau}$ into $g(T)$, $g(T)$ can be treated as a function of $M$, denoted by ${\cal G}(M)$. The second-order derivative of ${\cal G}(M)$ with respect to $M$ is given by 
	\begin{equation}\label{eq_partial_M}
	\begin{aligned}
	& \frac{\partial^2 {\cal G}(M)}{\partial M^2} = {\Theta} {\cal A}^{M} \ln^2\left({\cal A}\right) + 
	\frac{qL(\Delta s)^2\ln\left(\frac{1}{\delta}\right)\vartheta}{\epsilon^2 \left(\vartheta-{\cal A} \right)\left(U-1\right)}\\
	& \quad \times \left(\vartheta^{M} \ln^2\left(\vartheta\right) -{\cal A}^{M} \ln^2\left({\cal A}\right)+\left(\frac{{\cal A}}{\vartheta}\right)^{M}\ln^2\left(\frac{{\cal A}}{\vartheta}\right)  \right)\\
	& \quad + L \frac{\partial^2 {\cal H}\left(\frac{T}{M}\right)}{\partial M^2},
	\end{aligned}		 
	\end{equation}
	where
	\begin{equation*}
	\begin{aligned}
	\frac{\partial^2 {\cal H}\left(\frac{T}{M}\right)}{\partial M^2} & = \frac{\gamma T^2}{M^4} \left(\eta L + 1 \right)^{T/M}\ln^2\left(\eta L + 1 \right) \\
	& \quad + \frac{2\gamma T}{M^3} \left(\eta L + 1 \right)^{T/M}\ln\left(\eta L + 1 \right) - \frac{2\eta \theta T}{M^3} \geq 0,
	\end{aligned}
	\end{equation*}
	when $M \leq \frac{T\ln\left(\eta L + 1 \right) }{\ln\left(\eta/\ln\left(\eta L + 1 \right) \right)}$, or $\tau \geq \frac{\ln\left(\eta/\ln\left(\eta L + 1 \right) \right)}{\ln\left(\eta L + 1 \right)}$. 
	
	Like \eqref{eq_partial}, we see that \eqref{eq_partial_M} is positive when $\vartheta \geq {\cal A}$ and $\tau \geq \frac{\ln\left(\eta/\ln\left(\eta L + 1 \right) \right)}{\ln\left(\eta L + 1 \right)}$. Therefore, we have $\frac{\partial^2 g(M)}{\partial M^2} > 0$ if $\tau \geq \frac{\ln\left(\eta/\ln\left(\eta L + 1 \right) \right)}{\ln\left(\eta L + 1 \right)}$, and the upper bound is convex in~$M$.

	\bibliographystyle{IEEEtran}
	\bibliography{reference_DP}

\end{document}